\documentclass{article}
\usepackage{iclr2027_conference,times}

\usepackage{amsmath,amssymb,amsthm,mathtools,bm}
\usepackage{booktabs,array,tabularx,multirow}
\usepackage{algorithm}
\usepackage[noend]{algpseudocode}
\usepackage{graphicx}
\usepackage{thmtools}
\usepackage{thm-restate}
\usepackage{xcolor}
\usepackage{colortbl}
\usepackage{microtype}
\usepackage{hyperref}
\usepackage[nameinlink,capitalise,noabbrev]{cleveref}
\usepackage{url}
\usepackage{wrapfig}
\usepackage{float}
\usepackage{enumitem}
\usepackage{titletoc}
\usepackage[most]{tcolorbox}

\newtcolorbox{prompt}[1][]{
enhanced,
breakable,
colback=gray!4,
colframe=black!65,
coltitle=white,
fonttitle=\small\bfseries,
fontupper=\small,
boxrule=0.5pt,
arc=1.5pt,
left=5pt,
right=5pt,
top=4pt,
bottom=4pt,
before skip=6pt,
after skip=8pt,
#1
}

\declaretheorem[numberwithin=section,name=Theorem]{theorem}

\declaretheorem[sibling=theorem,name=Lemma]{lemma}
\declaretheorem[sibling=theorem,name=Corollary]{corollary}

\crefname{theorem}{Theorem}{Theorems}
\Crefname{theorem}{Theorem}{Theorems}
\crefname{proposition}{Proposition}{Propositions}
\Crefname{proposition}{Proposition}{Propositions}
\crefname{lemma}{Lemma}{Lemmas}
\Crefname{lemma}{Lemma}{Lemmas}
\crefname{corollary}{Corollary}{Corollaries}
\Crefname{corollary}{Corollary}{Corollaries}
\crefname{definition}{Definition}{Definitions}
\Crefname{definition}{Definition}{Definitions}
\crefname{assumption}{Assumption}{Assumptions}
\Crefname{assumption}{Assumption}{Assumptions}
\crefname{equation}{Equation}{Equations}
\Crefname{equation}{Equation}{Equations}
\crefname{algorithm}{Algorithm}{Algorithms}
\Crefname{algorithm}{Algorithm}{Algorithms}
\crefname{table}{Table}{Tables}
\Crefname{table}{Table}{Tables}
\crefname{figure}{Figure}{Figures}
\Crefname{figure}{Figure}{Figures}
\crefname{section}{Section}{Sections}
\Crefname{section}{Section}{Sections}

\newcommand{\E}{\mathbb{E}}

\newcommand{\1}{\mathbf{1}}

\newcommand{\abs}[1]{\left\lvert #1 \right\rvert}

\title{
Rethinking Training--Inference Mismatch in\\
LLM Reinforcement Learning:\\
Where It Arises and How to Correct It
}

\author{
\textbf{Tianrun Yu}$^1$,\ 
\textbf{Kaixiang Zhao}$^1$,\ 
\textbf{Shangzhe Li}$^2$,\ 
\textbf{Yuxiao Yang}$^2$,\ 
\textbf{Porter Jenkins}$^1$,\\
\textbf{Weitong Zhang}$^2$,\ 
\textbf{Taylor W. Killian}$^{1 *}$\\
$^1$Brigham Young University,\ 
$^2$University of North Carolina at Chapel Hill\\
{$^*$Corresponding author.
\tt \small
\{tianruny, kzhao2\}@byu.edu,\ 
tkillian@cs.byu.edu}\\
{\small }
}

\iclrfinalcopy

\begin{document}

\maketitle

\begingroup
\renewcommand{\thefootnote}{}
\footnotetext{Code is available at \url{https://github.com/kzhao5/CIS-RL}.}
\endgroup

\lhead{}

\begin{abstract}
We study training--inference mismatch in reinforcement learning with
verifiable rewards (RLVR) for large language models, where rollouts are
sampled by an inference engine while gradients are computed by a training engine,
and the two engines assign different probabilities to the same tokens.
To account for this discrepancy in policy updates, we introduce calibrated
importance sampling (CIS). CIS is motivated by an empirically supported
logit-displacement characterization that expresses the mismatch as an additive
displacement $\varepsilon_t$ in log-odds, determined by the per-logit
perturbation before the softmax, whose distribution is approximately invariant
to token confidence. This characterization motivates a confidence-aware
truncation: large positive displacements are truncated at a single constant
threshold, which maps back to an importance-ratio cap that tightens as token
confidence increases. Theoretically, we show that CIS replaces the unbounded
second moment that governs the error of exact importance sampling with a
term bounded by a constant, at the cost of a bias controlled by the truncated
excess. In evaluation across three mixture-of-experts models and five mathematical
reasoning benchmarks, CIS achieves the highest five-benchmark average on all
three models among the evaluated baselines. Diagnostic analyses show that CIS
places less truncation bias on low-confidence tokens than truncated importance
sampling, while upward clipping of small importance weights reduces held-out
accuracy.
\end{abstract}

\section{Introduction}
\label{sec:introduction}

Reinforcement learning with verifiable rewards (RLVR) has become a dominant
paradigm for improving the reasoning abilities of large language models
(LLMs) \citep{shao2024deepseekmath, deepseekai2025deepseekr1, yu2025dapo}.
To improve throughput, current RL post-training frameworks
decouple rollout generation from gradient computation and assign them to an
inference engine and a training engine,
respectively~\citep{sheng2025hybridflow,hu2025openrlhf,fu2025areal}.
Although the
two engines load the same parameters, differences in their numerical
implementations lead them to assign different probabilities to the same token.
As a result, training that is nominally on-policy becomes
off-policy~\citep{yao2025offpolicy,qi2025fp16,zheng2025stabilizing}.
This
\emph{training--inference mismatch} is particularly severe for mixture-of-experts
models, where it can cause training instability or even
collapse~\citep{ling2025everystep,zheng2025gspo,ma2025r3,zheng2025stabilizing}.

The standard remedy for this mismatch is importance sampling, which reweights
the gradient of each token by the ratio between the probabilities that the
training and inference engines assign to the sampled
token~\citep{yao2025offpolicy,liu2025rlcollapse,zheng2025stabilizing}.
However, occasional sharp disagreements between the two engines on individual
tokens can make this importance ratio extremely
large~\citep{zhao2025icepop,ma2025r3}.
Exact correction can therefore yield gradient estimates with uncontrolled
variance~\citep{ionides2008truncated,metelli2018pois,liu2025rlcollapse}.
Truncating the importance sampling factor controls this variance by introducing the bias--variance trade-off. Understanding this bias--variance trade-off and designing better truncation methods has emerged as a new question to answer~\citep{ionides2008truncated,liu2025rlcollapse}.

\begin{figure}[t]
    \centering
    \includegraphics[width=0.9\linewidth]{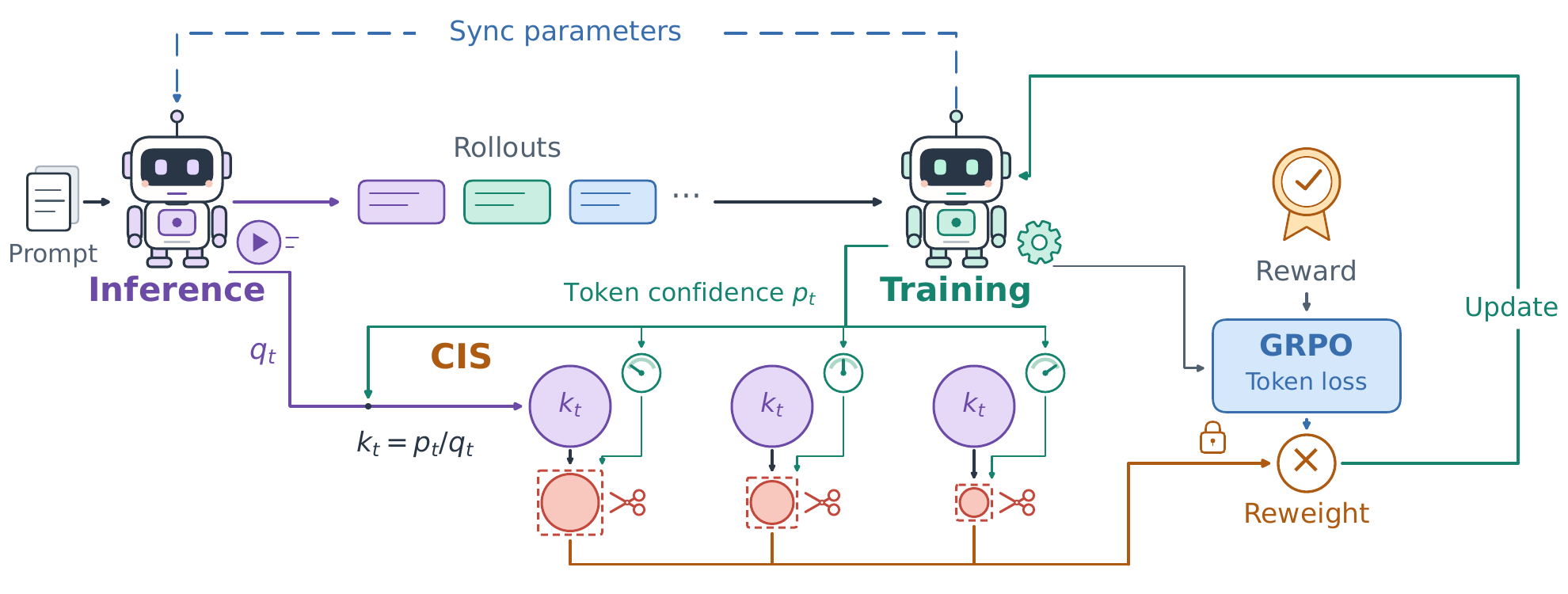}
    \vspace{-1.0\baselineskip}
    \caption{
        The inference and training engines assign different probabilities
        to the same tokens.
        CIS truncates the importance ratios with confidence-dependent
        caps, which are tighter for more confident tokens.
        The truncated ratios reweight the token-level
        GRPO loss~\citep{shao2024deepseekmath}.
    }
    \label{fig:cis_pipeline}
\end{figure}

Existing methods largely rely on heuristic rules to truncate or mask this
ratio. At the token level, TIS caps
the ratio at a fixed upper bound~\citep{yao2025offpolicy,qi2025fp16},
IcePop masks tokens whose ratio
falls outside a fixed interval~\citep{zhao2025icepop,ling2025everystep},
and KPop masks tokens by thresholding the KL divergence between the
probabilities of the inference and training engines~\citep{guo2026kpop,ling2026lingring}.
At the sequence level, existing methods truncate or
reject the ratio of the entire response~\citep{liu2025rlcollapse},
or replace the token-level
objective with a sequence-level one~\citep{zheng2025gspo}.
Other work addresses the problem through numerical precision and runs the
entire pipeline in FP16~\citep{qi2025fp16}. Although these methods are
effective to varying degrees, all of their rules act directly on the
importance ratio or on quantities computed directly from the probabilities
of the two engines.
This ratio, however, mixes two factors: the actual discrepancy between the two
engines and the confidence of the token itself. A discrepancy of the same size
moves the ratio far from one on low-confidence tokens but barely changes it
on high-confidence tokens. As a result, fixed thresholds act mainly on
low-confidence tokens~\citep{guo2026kpop,ling2026lingring}, and the
truncation bias concentrates on these tokens. This motivates our
question:

\begin{center}
\emph{Can we design a principled mismatch correction grounded in the \\
statistical structure of the training--inference mismatch itself?}
\end{center}

We address this question by proposing
\textbf{CIS} (\underline{\textbf{C}}alibrated
\underline{\textbf{I}}mportance \underline{\textbf{S}}ampling), which is
illustrated in \Cref{fig:cis_pipeline}. CIS first separates the two factors in
the importance ratio, which can be written exactly as a function of the token
confidence and a component that reflects only the discrepancy between the two
engines. Our measurements show that the distribution of this discrepancy
component varies only slightly with token confidence, so anomalous
discrepancies can be identified with the same threshold at every confidence
level. CIS applies a single constant threshold to this component, which
truncates anomalously large discrepancies at their source. When mapped back to
the ratio, this threshold imposes a cap on the importance weight of each token,
which depends on the token confidence and becomes tighter as confidence
increases. This design determines where the truncation bias falls.
Discrepancies of the same size are treated in the same way at every confidence
level, so the truncation bias no longer concentrates on low-confidence tokens.
Our contributions are as follows:

\begin{itemize}[leftmargin=*]
\item We decompose the importance ratio into token confidence and a
discrepancy component, and find in large-scale analyses that the
distribution of the discrepancy component barely varies with token confidence.
With mixture-of-experts models, routing disagreement between engines
substantially amplifies the heavy tail of this distribution.

\item Building on this characterization, we propose CIS, which truncates the
discrepancy component rather than the ratio, yielding a ratio cap that
tightens with token confidence. Theoretically, we prove that CIS bounds the
potentially explosive second-moment term of exact correction, at the cost of a
bias controlled by the truncated excess.

\item Experiments on three mixture-of-experts models and five mathematical
reasoning benchmarks validate the effectiveness of CIS. Diagnostic analyses
show that, unlike a fixed ratio cap, CIS does not concentrate its
truncation bias on low-confidence tokens and attains a lower overall bias.
\end{itemize}

\section{Preliminaries}
\label{sec:prelim}

\subsection{Training--inference mismatch}
\label{sec:prelim_mismatch}

In current RL post-training frameworks, rollouts are generated by an inference
engine (e.g., vLLM and SGLang), while gradients are computed by a separate
training engine (e.g., FSDP and Megatron). Both engines load the same parameters
$\theta$, but differences in kernel implementations, reduction orders, and
numerical precision lead them to assign different probabilities to the same
token. We write $h_t := (x, y_{<t})$ for the prefix at position $t$. Then
$\pi_{\mathrm{train}}(y_t\mid h_t;\theta)\ne\pi_{\mathrm{infer}}(y_t\mid h_t;\theta)$,
so rollouts that are nominally on-policy are in fact off-policy with respect to
the training engine. The problem is more severe for mixture-of-experts models,
where routing is a discrete top-$k$ decision and a small perturbation can change
which experts are activated.

For the token-level GRPO surrogate evaluated on the observed rollout prefixes,
correcting the conditional distribution of each sampled token introduces an
importance ratio:
\begin{equation}
\mathcal J(\theta)
=
\E_{x\sim\mathcal D,\;\{y_i\}_{i=1}^{G}\sim\pi_{\mathrm{infer}}(\cdot\mid x;\theta_{\mathrm{old}})}
\left[
\frac{1}{G}\sum_{i=1}^{G}\frac{1}{\abs{y_i}}\sum_{t=1}^{\abs{y_i}}
k_{i,t}\,
\min\Big(
\rho_{i,t}\hat A_{i,t},\;
\operatorname{clip}\big(\rho_{i,t}\big)\hat A_{i,t}
\Big)
\right],
\label{eq:objective}
\end{equation}
where $h_{i,t}=(x,y_{i,<t})$, $\theta_{\mathrm{old}}$ is the parameter snapshot at
rollout time, $\hat A_{i,t}$ is the group-normalized advantage, and
$\operatorname{clip}$ restricts $\rho_{i,t}$ to
$[1-c_{\mathrm{low}},\,1+c_{\mathrm{high}}]$. Let
$p_{i,t}=\pi_{\mathrm{train}}(y_{i,t}\mid h_{i,t};\theta_{\mathrm{old}})$ and
$q_{i,t}=\pi_{\mathrm{infer}}(y_{i,t}\mid h_{i,t};\theta_{\mathrm{old}})$ be the
probabilities that the two engines assign to the sampled token. The update ratio
is $\rho_{i,t}=\pi_{\mathrm{train}}(y_{i,t}\mid h_{i,t};\theta)/p_{i,t}$, and the
mismatch ratio is $k_{i,t}=p_{i,t}/q_{i,t}$. The mismatch ratio $k_{i,t}$ corrects
only the next-token distribution given the observed prefix $h_{i,t}$; it does not
reweight the distribution of the prefix itself. The update ratio $\rho_{i,t}$ is
handled by the standard PPO clipping, while the mismatch ratio $k_{i,t}$ carries
the discrepancy between the two engines and is the focus of the rest of this
paper.

\subsection{Correcting the mismatch}
\label{sec:prelim_correction}

The most direct way to correct the training--inference mismatch is to use the
exact ratio $k_{i,t}$ in \Cref{eq:objective}. This correction is unbiased, but on
the few tokens where the two engines disagree sharply, the ratio becomes very
large and the gradient estimate has uncontrolled variance. Existing methods
therefore truncate or mask $k_{i,t}$ and accept some bias in exchange for lower
variance.

At the token level, TIS replaces $k_{i,t}$ with $\min(k_{i,t}, C)$ for a fixed cap
$C$~\citep{yao2025offpolicy}. IcePop instead uses
$k_{i,t}\,\mathbf{1}\{\alpha\le k_{i,t}\le\beta\}$, which masks any token whose
ratio falls outside a fixed interval~\citep{zhao2025icepop,ling2025everystep}.
KPop masks a token when the binary KL divergence between the probabilities of
the two engines exceeds a fixed threshold~\citep{guo2026kpop,ling2026lingring}.
At the sequence level, Seq-TIS and Seq-MIS apply the same cap or mask to the
product of the ratios over a response~\citep{liu2025rlcollapse}, and GSPO
optimizes a sequence-level objective instead~\citep{zheng2025gspo}. A different
line of work leaves the reweighting unchanged and reduces the discrepancy
itself: FP16 training raises numerical precision~\citep{qi2025fp16}, and routing
replay makes the two engines select the same experts~\citep{ma2025r3}. All of
these approaches are effective to some extent. However, the thresholds of the
reweighting methods either act on $k_{i,t}$ directly and treat every token in the
same way, or are not derived from the statistics of the mismatch itself. In
Section~\ref{sec:cis}, we start from the logit perturbation between the two
engines, characterize the structure of the mismatch, and build our correction on
this characterization.

\section{Calibrated Importance Sampling}
\label{sec:cis}
\subsection{From logit perturbation to log-odds displacement}
\label{sec:logit_model}

We measure, on large-scale rollouts, the probabilities that the
two engines assign to the same sampled token (experimental details in
\Cref{app:mm_protocol}). On mixture-of-experts models, the top-$k$
selection of the gating network is a discrete decision, so a small numerical
difference can make the two engines activate different experts at some layer,
and this difference propagates through all subsequent layers. These differences
accumulate layer by layer and ultimately appear as different logits output by
the two engines under the same parameters and the same prefix. Write
$\mathbf z=(z_1,\dots,z_V)\in\mathbb{R}^{V}$ for the logit vector of the
training engine at the prefix $h_t$, and $\mathbf z+\boldsymbol\delta$ for that
of the inference engine, where
$\boldsymbol\delta=(\delta_1,\dots,\delta_V)\in\mathbb{R}^{V}$ is the per-logit
perturbation between the two engines. Then
$p_t=\pi_{\mathrm{train}}(y_t\mid h_t;\theta_{\mathrm{old}})
=\operatorname{softmax}(\mathbf z)_{y_t}$ and
$q_t=\pi_{\mathrm{infer}}(y_t\mid h_t;\theta_{\mathrm{old}})
=\operatorname{softmax}(\mathbf z+\boldsymbol\delta)_{y_t}$. Substituting the
two softmaxes into the mismatch ratio $k_t:=p_t/q_t$ yields the
following identity.

\begin{wrapfigure}[20]{r}{0.48\textwidth}
\vspace{-1.0\baselineskip}
\centering
\includegraphics[width=0.48\textwidth]{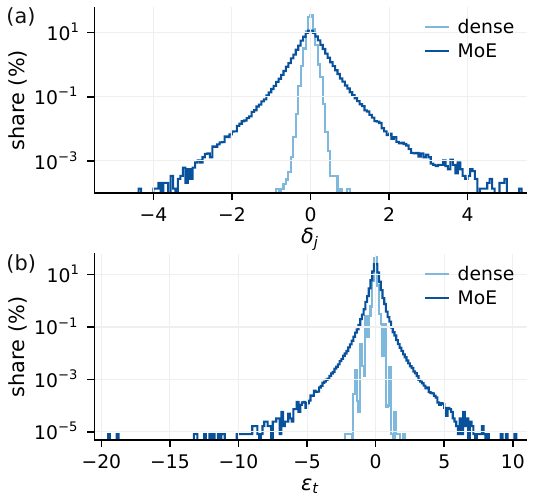}
\vspace{-1.25\baselineskip}
\caption{(a) Per-logit perturbation $\delta_j$ and (b) log-odds displacement
$\varepsilon_t$, MoE vs.\ dense.}
\label{fig:mismatch}
\vspace{-2.0\baselineskip}
\end{wrapfigure}

\begin{lemma}[Ratio--displacement identity]
\label{lem:ratio_displacement}
Let $w_j:=e^{z_j}\big/\sum_{j'\ne y_t}e^{z_{j'}}$ for $j\ne y_t$. Then the
mismatch ratio satisfies
\begin{equation}
k_t
=\frac{\operatorname{softmax}(\mathbf z)_{y_t}}
      {\operatorname{softmax}(\mathbf z+\boldsymbol\delta)_{y_t}}
=p_t+(1-p_t)\exp(\varepsilon_t),
\qquad
k_t-1=(1-p_t)\big(\exp(\varepsilon_t)-1\big),
\label{eq:k_geometry}
\end{equation}
where $\varepsilon_t:=-\delta_{y_t}+\log\sum_{j\ne y_t} w_j \exp(\delta_j)$.
\end{lemma}

The proof is given in \Cref{app:lemmas}. We call
$\varepsilon_t$ the log-odds displacement of the inference engine relative to
the training engine on the current token. From \Cref{eq:k_geometry},
$\varepsilon_t=\log\frac{p_t}{1-p_t}-\log\frac{q_t}{1-q_t}$, that is,
$\varepsilon_t$ is exactly the difference between the log-odds that the two
engines assign to the sampled token. In the log-odds coordinate, the discrepancy
between the two engines is therefore an additive displacement determined
entirely by the perturbation $\boldsymbol\delta$, and it can be computed
directly from the log-probabilities recorded by the two engines.
\Cref{fig:mismatch} shows the measured $\boldsymbol\delta$ and
$\varepsilon_t$. On the dense control, $\boldsymbol\delta$ is concentrated near
zero and bounded, consistent with bf16 rounding error, and $\varepsilon_t$ is
correspondingly concentrated tightly around zero. On the MoE, the body of
$\boldsymbol\delta$ is likewise concentrated near zero, but is distinguished
by a pronounced heavy tail. We analyze the source of this tail in
\Cref{app:mm_routing,app:mm_route_control}: rounding error accounts only for
the body, while routing disagreement between the two engines is a major
amplifier of the tail, and the tail of $\boldsymbol\delta$ becomes heavier as
the two engines select different experts at more layers. Since
$\varepsilon_t:=-\delta_{y_t}+\log\sum_{j\ne y_t} w_j \exp(\delta_j)$, the heavy tail of $\boldsymbol\delta$
carries over to $\varepsilon_t$, so the log-odds displacement on the MoE is
heavy-tailed as well. Given $p_t$, by \Cref{eq:k_geometry}, the ratio
is a deterministic, increasing function of $\varepsilon_t$: the mismatch enters
the importance weight through the displacement factor
$e^{\varepsilon_t}$, rescaled by $1-p_t$. The same displacement therefore moves
the ratio far from one on low-confidence tokens, whereas on high-confidence
tokens the factor $1-p_t$ compresses it and the ratio barely moves.

\subsection{Why exact correction is unstable}
\label{sec:why_correction}

A gradient update uses $n$ tokens, indexed by $m$. Write the payload of the
$m$-th token as
$\ell'_m~=~\nabla_\theta\min\big(\rho_m\hat A_m,\operatorname{clip}(\rho_m)\hat A_m\big)$;
the mismatch ratio $k_m$ is the importance sampling factor used to scale this gradient.
Since $k_m$ does not depend on $\theta$, the gradient of the token-level
surrogate objective and its estimator are
\begin{equation*}
G_0:=\nabla_\theta\mathcal J=\E[k\,\ell'],
\qquad
\widehat G:=\frac{1}{n}\sum_{m=1}^{n}k_m\,\ell'_m ,
\end{equation*}
where the expectation is over the prompt, rollout, and numerical
discrepancy between the two engines for each token. $G_0$ corrects the
conditional distribution of each token given its prefix, not the distribution
of the prefix itself. We measure the quality of the estimate by the
mean-squared error $\E\big\|\widehat G-G_0\big\|_2^2$.

\begin{wrapfigure}{r}{0.5\textwidth}
\vspace{-1.0\baselineskip}
\centering
\includegraphics[width=0.42\textwidth]{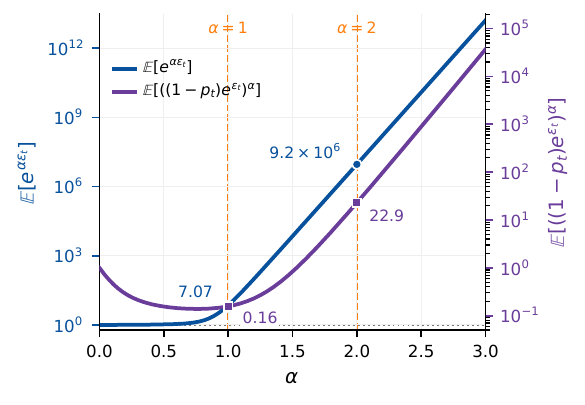}
\vspace{-1.4\baselineskip}
\caption{$\E[e^{\alpha\varepsilon_t}]$ (left axis) and
$\E[((1-p_t)e^{\varepsilon_t})^{\alpha}]$ (right axis) on the measured
displacements of the trained MoE.}
\label{fig:mgf}
\vspace{-0.8\baselineskip}
\end{wrapfigure}

\begin{theorem}[Informal]
\label{thm:exact_risk}
Assume $\|\ell'_t\|_2\le b$ and that the contributions of different tokens are
independent (for correlated contributions, replace $n$ by $n_{\mathrm{eff}}$;
see \Cref{app:proofs}). Then
\begin{equation}
\E\big\|\widehat G-G_0\big\|_2^2
\;\le\;
\frac{b^2}{n}\,\E\Big[p_t^2+2p_t(1-p_t)\exp(\varepsilon_t)+(1-p_t)^2\,e^{2\varepsilon_t}\Big].
\label{eq:exact_risk}
\end{equation}
\end{theorem}

\Cref{fig:mgf} shows both moments on the displacements of the MoE after RL training (\Cref{app:mm_moment}). The raw moment $\E[e^{\alpha\varepsilon_t}]$ is only
$7.07$ at $\alpha=1$ but reaches $9.2\times10^{6}$ at $\alpha=2$. The weighted
moment $\E[((1-p_t)e^{\varepsilon_t})^{\alpha}]$, which is the quantity that
enters \Cref{eq:exact_risk}, is compressed by the factor $(1-p_t)^{\alpha}$ but
follows the same pattern: it rises from $0.16$ at $\alpha=1$ to $22.9$ at
$\alpha=2$. Since $p_t^2\le1$ and the cross term is at most $2\times0.16$, the
bound (\Cref{eq:exact_risk}) is dominated by the second-moment term and is more
than twenty times the corresponding bound without correction, for which the
expectation equals one. Exact correction is therefore unbiased, but its variance
is driven by the upper tail of $\varepsilon_t$.

This changes the problem from whether the mismatch should be corrected to
how the exact correction should be truncated. Some bias must be accepted in
exchange for controlling its second moment. The next section constructs such
a truncation directly in the coordinate from which the mismatch arises.
\subsection{Correction in the displacement coordinate}
\label{sec:optimal_operator}

\begin{wrapfigure}[20]{r}{0.42\textwidth}
\centering
\includegraphics[width=0.42\textwidth]{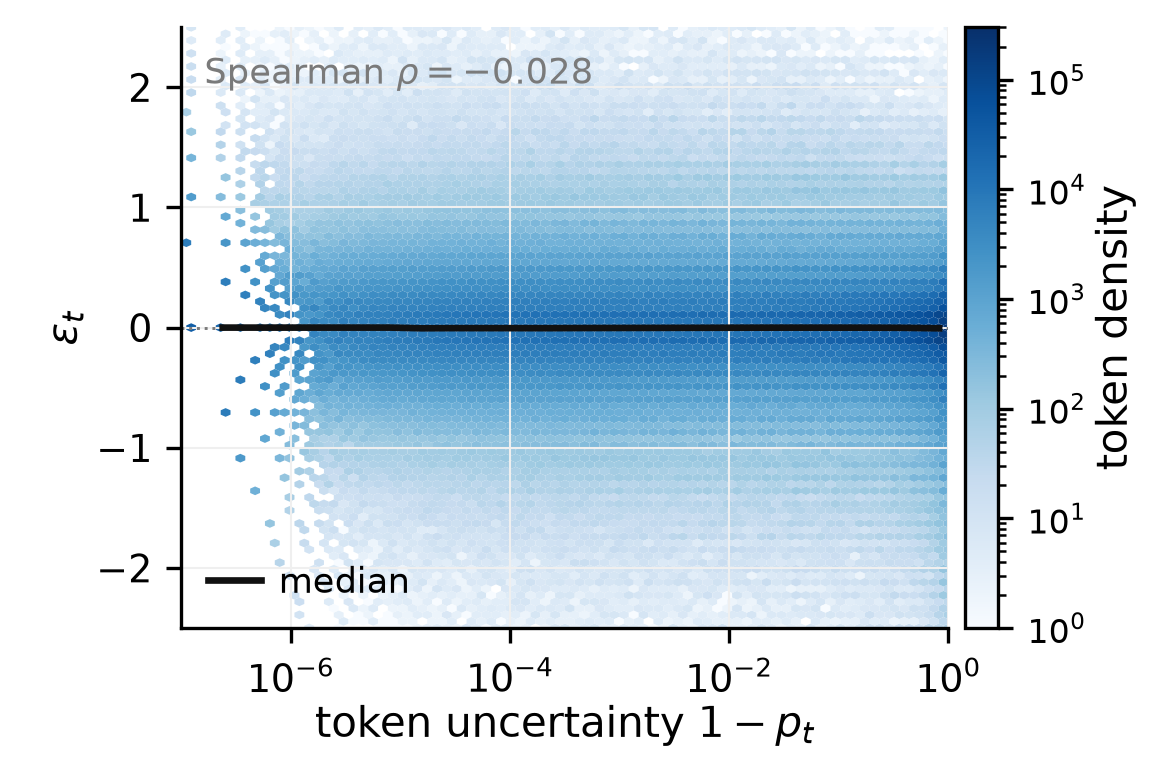}
\vspace{-1.4\baselineskip}
\caption{$\varepsilon_t$ against token uncertainty $1-p_t$ on the base MoE; the
black line is the median.}
\label{fig:eps_vs_p}
\end{wrapfigure}

Given $p_t$, the probability of generating a token, the mismatch enters the ratio only through
$e^{\varepsilon_t}$, by \Cref{eq:k_geometry}:
\[
k_t=p_t+(1-p_t)\exp(\varepsilon_t).
\]
The variance problem is intrinsically one-sided. If
$e^{\varepsilon_t}\le1$, then
\[
p_t<k_t\le1,
\]
so the importance weight on this side satisfies $k_t^2\le1$.
In contrast, when $\exp(\varepsilon_t)>1$, the ratio exceeds one and can grow
without bound as the positive displacement increases. Thus the second-moment
explosion identified in \Cref{sec:why_correction} can only arise from the
upper tail. Modifying the lower side does not address this source of
unbounded variance and instead introduces additional deviation from the exact
correction. We therefore truncate only large positive displacements. It remains to choose where to truncate the upper tail. 

The displacement itself is only weakly related to confidence.
\Cref{fig:eps_vs_p} plots $\varepsilon_t$ against $1-p_t$ on the base MoE (\Cref{app:mm_confidence}):
across six orders of magnitude of uncertainty, the median stays at zero and the spread changes
by less than a factor of two with a Spearman rank correlation of only $-0.028$.
A displacement of a given amount is therefore approximately equally anomalous
across confidence levels. This suggests placing a single upper threshold
directly in the displacement coordinate with
$\exp(\varepsilon_t)\le1+\lambda$ where threshold $\lambda\ge0$.

Truncating $e^{\varepsilon_t}$ at $1+\lambda$ and mapping back through \Cref{eq:k_geometry} gives
\begin{equation}
f_{\lambda}(k_t,p_t)
\;=\;
p_t+(1-p_t)\min\big\{e^{\varepsilon_t},\,1+\lambda\big\}
\;=\;
\min\big\{k_t,\;1+\lambda(1-p_t)\big\}.
\label{eq:cis_operator}
\end{equation}
{\bf \Cref{eq:cis_operator} introduces \emph{calibrated importance sampling} (CIS), our proposed correction for training--inference mismatch.}
The factor $1-p_t$ in the cap is not a design choice: it is the factor by which \Cref{eq:k_geometry} maps displacements to ratios. The two forms in \Cref{eq:cis_operator} are identical, token by token; we implement the
right-hand form as it only requires the difference of the two
log-probabilities, which remains finite when a probability rounds to one in bf16,
whereas $\operatorname{logit}p_t$ and $\operatorname{logit}q_t$ do not.

For comparison, a fixed ratio cap $k_t\le C$ corresponds through \Cref{eq:k_geometry} to
\begin{equation}
\exp(\varepsilon_t)
\le
1+\frac{C-1}{1-p_t},
\label{eq:fixed_cap_eps}
\end{equation}
whose displacement threshold becomes increasingly permissive as confidence
grows. CIS instead applies the same upper displacement threshold at every
confidence level, and the confidence-dependent cap in \Cref{eq:cis_operator} follows automatically from the mapping back to
ratio space.

\begin{theorem}[Informal]
\label{thm:cis_risk}
Assume $\|\ell'_t\|_2\le b$ and that the contributions of different tokens are
independent. Then for any finite $\lambda\ge0$,
\begin{equation}
\begin{aligned}
\E\big\|\widehat G_{f_{\lambda}}-G_0\big\|_2^2
\;\le\;& b^2\Big(\E\big[(1-p_t)(e^{\varepsilon_t}-1-\lambda)_+\big]\Big)^2+\frac{b^2}{n}\,\E\Big[\big(p_t+(1-p_t)\min\{e^{\varepsilon_t},1+\lambda\}\big)^2\Big].
\end{aligned}
\label{eq:cis_risk}
\end{equation}
\end{theorem}

Compared with \Cref{eq:exact_risk}, the unbounded second moment
$\E[e^{2\varepsilon_t}]$ is replaced by
$\E\big[\min(e^{\varepsilon_t},1+\lambda)^2\big]\le(1+\lambda)^2$,
at the price of a bias from the truncated excess; $\lambda$ controls this
trade-off, and $\lambda\to\infty$ recovers exact correction. Any cap with a
finite maximum admits a similar bound, so the advantage of CIS over a
fixed ratio cap lies not in this bound but in where the bias falls:
CIS spreads the truncation bias more evenly across confidence levels and
attains a lower overall bias (\Cref{sec:experiments-ablation}).

\begin{wrapfigure}{r}{0.5\textwidth}
\vspace{-\baselineskip}
\begin{minipage}{0.5\textwidth}
\begin{algorithm}[H]
\caption{Calibrated importance sampling (CIS)}
\label{alg:cis}
\begin{algorithmic}[1]
\Require $\ell^{\mathrm{train}}_m$, $\ell^{\mathrm{infer}}_m$, $\ell'_m$ for
the $n$ tokens of the batch; threshold $\lambda$; floor $\kappa$
\For{$m=1,\dots,n$}
  \State $p_m\gets\exp(\ell^{\mathrm{train}}_m)$,
  \Statex \hspace{\algorithmicindent}$k_m\gets
    \exp(\ell^{\mathrm{train}}_m-\ell^{\mathrm{infer}}_m)$
  \State $\varphi_m\gets\max(1-p_m,\,\kappa)$
  \State $f_m\gets\min\big\{k_m,\;1+\lambda\varphi_m\big\}$
\EndFor
\State \Return $\widehat G_{\mathrm{CIS}}
=\frac{1}{n}\sum_{m=1}^{n}\operatorname{stopgrad}(f_m)\,\ell'_m$
\end{algorithmic}
\end{algorithm}
\end{minipage}
\end{wrapfigure}

\paragraph{Algorithm.}
We treat $\lambda$ as a hyperparameter; \Cref{sec:experiments-tuning} shows
that every positive threshold outperforms the uncorrected run. Applying
\Cref{eq:cis_operator} in practice raises one issue: as $1-p_t$ becomes very
small, the cap $1+\lambda(1-p_t)$ approaches one and becomes narrower than the
storage resolution of the log-probabilities. The truncated tokens then have
$|\log k_t|$ at the level of that resolution, so the truncation acts on
rounding error rather than on the mismatch; without a remedy, the average
accuracy drops from $34.78$ to $29.02$ (\Cref{sec:experiments-ablation}). We therefore floor $1-p_t$ at the storage
resolution $\kappa=5\times10^{-3}$, replacing it by
$\varphi_t=\max(1-p_t,\kappa)$, which gives \Cref{alg:cis}. For tokens with
$1-p_t\ge\kappa$, \Cref{alg:cis} applies exactly the displacement threshold
$\exp(\varepsilon_t)\le1+\lambda$; below the floor, the threshold becomes
$e^{\varepsilon_t}\le1+\lambda\kappa/(1-p_t)$, which is looser. The correction adds
one elementwise pass and no forward or backward computation. The
inference-side log-probability must be the raw value recorded at the sampling
step, and the weight must be detached.

\section{Experiments}
\label{sec:experiments}

We organize the experiments around three questions.
\textbf{RQ1}: Does CIS improve held-out performance over existing
training--inference mismatch corrections?
\textbf{RQ2}: Why does CIS work, and which design choices matter?
\textbf{RQ3}: How sensitive is CIS to its hyperparameters?
Section~\ref{sec:experiments-setup} describes the shared setup.
Implementation details, training dynamics, additional diagnostics, and
out-of-domain evaluations are provided in \Cref{app:experimental,app:discussion}.

\subsection{Setup}
\label{sec:experiments-setup}

We evaluate RLVR on three MoE models:
Qwen1.5-MoE-A2.7B-Chat~\citep{qwen2024qwen15moe},
DeepSeek-V2-Lite~\citep{deepseekai2024deepseekv2}, and
Qwen3-30B-A3B~\citep{yang2025qwen3}.
All models are trained on GSM8K~\citep{cobbe2021gsm8k} and evaluated with
greedy pass@1 on five mathematical reasoning benchmarks:
GSM8K~\citep{cobbe2021gsm8k},
MATH500~\citep{hendrycks2021math,lightman2024verify},
SVAMP~\citep{patel2021svamp},
Minerva-Math~\citep{lewkowycz2022minerva},
and OlympiadBench~\citep{he2024olympiadbench}.
All methods share the same training recipe and code
path and differ only in how they handle the training--inference mismatch.
Unless stated otherwise, we report mean and standard deviation over three
training seeds.
We group the baselines by where they act:
\emph{token-level} (Exact Ratio~\citep{yao2025offpolicy,liu2025rlcollapse},
TIS~\citep{yao2025offpolicy},
IcePop~\citep{zhao2025icepop,ling2025everystep}, and
KPop$^\dagger$~\citep{guo2026kpop,ling2026lingring}),
\emph{sequence-level} (Seq-TIS and Seq-MIS~\citep{liu2025rlcollapse},
and GSPO~\citep{zheng2025gspo}), and
\emph{numerical} (FP16~\citep{qi2025fp16}).
The baseline ``No Correction'' provides the uncorrected reference, while CIS
belongs to the token-level class.
Full implementation details, hyperparameters, and evaluation protocols are
given in Appendix~\ref{app:experimental}.

\subsection{RQ1: Does CIS improve held-out performance?}
\label{sec:experiments-results}

\begin{table}[t]
\centering
\caption{Held-out accuracy (\%) on five math benchmarks after RL on GSM8K.
Mean $\pm$ standard deviation over three seeds; bold marks the best mean
per column within each model.}
\label{tab:main}
\scriptsize\setlength{\tabcolsep}{2pt}
\renewcommand{\arraystretch}{0.88}
\setlength{\aboverulesep}{0.7pt}\setlength{\belowrulesep}{0.7pt}
\begin{tabularx}{\linewidth}{cl|*{5}{>{\centering\arraybackslash}X}|>{\centering\arraybackslash}X}
\toprule
Model & Method & GSM8K & MATH500 & SVAMP & Minerva & Olympiad & Avg \\
\midrule

\multirow{11}{*}{\rotatebox[origin=c]{90}{\scriptsize Qwen1.5-MoE-A2.7B}}
& Base model (no RL) & $55.50$ & $21.20$ & $59.00$ & $4.04$ & $4.90$ & $28.93$ \\
& No correction & $61.21_{\pm1.17}$ & $21.33_{\pm1.40}$ & $64.33_{\pm1.73}$ & $3.68_{\pm1.27}$ & $4.40_{\pm0.45}$ & $30.99_{\pm0.22}$ \\
\cmidrule(lr){2-8}
& Exact ratio & $63.18_{\pm2.97}$ & $22.00_{\pm1.22}$ & $67.11_{\pm3.75}$ & $5.51_{\pm0.64}$ & $4.60_{\pm0.15}$ & $32.48_{\pm1.21}$ \\
& KPop$^\dagger$ & $64.80_{\pm1.29}$ & $23.67_{\pm1.29}$ & $68.67_{\pm3.71}$ & $5.15_{\pm1.47}$ & $5.24_{\pm0.67}$ & $33.50_{\pm1.04}$ \\
& IcePop & $66.16_{\pm0.96}$ & $24.33_{\pm1.03}$ & $69.78_{\pm1.02}$ & $5.51_{\pm0.64}$ & $5.09_{\pm1.11}$ & $34.18_{\pm0.55}$ \\
& TIS & $61.03_{\pm5.11}$ & $20.73_{\pm3.92}$ & $66.78_{\pm5.09}$ & $4.54_{\pm0.76}$ & $3.91_{\pm0.31}$ & $31.40_{\pm2.47}$ \\
\cmidrule(lr){2-8}
& Seq-TIS & $65.93_{\pm0.19}$ & $22.93_{\pm0.61}$ & $63.56_{\pm1.39}$ & $3.68_{\pm1.60}$ & $5.29_{\pm0.89}$ & $32.28_{\pm0.86}$ \\
& Seq-MIS & $60.98_{\pm1.50}$ & $22.07_{\pm0.50}$ & $63.56_{\pm1.17}$ & $4.41_{\pm0.64}$ & $4.75_{\pm0.26}$ & $31.15_{\pm0.67}$ \\
& GSPO & $60.10_{\pm2.77}$ & $17.47_{\pm0.50}$ & $58.56_{\pm7.73}$ & $5.02_{\pm0.56}$ & $3.36_{\pm0.31}$ & $28.90_{\pm1.25}$ \\
& FP16 & $60.02_{\pm1.16}$ & $22.67_{\pm0.46}$ & $62.78_{\pm1.54}$ & $5.02_{\pm0.21}$ & $4.60_{\pm0.45}$ & $31.02_{\pm0.39}$ \\
\cmidrule(lr){2-8}
\rowcolor{blue!8}\cellcolor{white}
& CIS (ours) & $\mathbf{67.88}_{\pm1.64}$ & $\mathbf{24.53}_{\pm1.03}$ &
$\mathbf{70.22}_{\pm1.35}$ & $\mathbf{5.88}_{\pm0.64}$ &
$\mathbf{5.39}_{\pm0.62}$ & $\mathbf{34.78}_{\pm0.25}$ \\

\midrule
\multirow{11}{*}{\rotatebox[origin=c]{90}{\scriptsize DeepSeek-V2-Lite}}
& Base model (no RL) & $70.66$ & $25.40$ & $67.67$ & $6.25$ & $6.08$ & $35.21$ \\
& No correction & $71.06_{\pm0.83}$ & $23.47_{\pm0.64}$ & $70.89_{\pm2.91}$ & $7.23_{\pm1.49}$ & $6.48_{\pm0.56}$ & $35.83_{\pm1.08}$ \\
\cmidrule(lr){2-8}
& Exact ratio & $71.17_{\pm0.27}$ & $24.67_{\pm0.64}$ & $70.22_{\pm1.84}$ & $6.37_{\pm1.49}$ & $4.95_{\pm0.56}$ & $35.47_{\pm0.34}$ \\
& KPop$^\dagger$ & $71.14_{\pm0.59}$ & $26.33_{\pm0.64}$ & $70.89_{\pm2.55}$ & $6.62_{\pm1.10}$ & $6.13_{\pm0.82}$ & $36.22_{\pm0.74}$ \\
& IcePop & $72.48_{\pm0.69}$ & $25.60_{\pm0.92}$ & $71.44_{\pm1.02}$ & $7.72_{\pm1.60}$ & $6.43_{\pm0.76}$ & $36.73_{\pm0.73}$ \\
& TIS & $72.55_{\pm0.84}$ & $24.80_{\pm1.00}$ & $69.33_{\pm1.76}$ & $7.72_{\pm0.97}$ & $5.69_{\pm1.12}$ & $36.02_{\pm0.51}$ \\
\cmidrule(lr){2-8}
& Seq-TIS & $71.87_{\pm0.50}$ & $24.53_{\pm0.70}$ & $70.89_{\pm0.84}$ & $6.74_{\pm1.18}$ & $5.74_{\pm0.45}$ & $35.95_{\pm0.29}$ \\
& Seq-MIS & $70.61_{\pm0.46}$ & $25.93_{\pm1.01}$ & $70.78_{\pm2.59}$ & $6.86_{\pm0.56}$ & $6.33_{\pm0.45}$ & $36.10_{\pm0.71}$ \\
& GSPO & $72.50_{\pm0.86}$ & $25.93_{\pm0.70}$ & $71.33_{\pm1.00}$ & $6.37_{\pm1.66}$ & $6.18_{\pm0.52}$ & $36.47_{\pm0.43}$ \\
& FP16 & $71.80_{\pm0.60}$ & $24.47_{\pm0.70}$ & $69.89_{\pm2.17}$ & $7.97_{\pm1.66}$ & $6.23_{\pm0.68}$ & $36.07_{\pm0.61}$ \\
\cmidrule(lr){2-8}
\rowcolor{blue!8}\cellcolor{white}
& CIS (ours) & $\mathbf{72.58}_{\pm0.76}$ & $\mathbf{26.47}_{\pm0.99}$ &
$\mathbf{71.89}_{\pm2.04}$ & $\mathbf{8.21}_{\pm1.66}$ &
$\mathbf{6.68}_{\pm0.51}$ & $\mathbf{37.16}_{\pm0.73}$ \\

\midrule
\multirow{11}{*}{\rotatebox[origin=c]{90}{\scriptsize Qwen3-30B-A3B}}
& Base model (no RL) & $95.22$ & $74.00$ & $93.33$ & $31.25$ & $45.40$ & $67.84$ \\
& No correction & $94.79_{\pm0.32}$ & $72.40_{\pm0.72}$ & $94.56_{\pm1.17}$ & $33.33_{\pm0.93}$ & $46.24_{\pm0.52}$ & $68.26_{\pm0.27}$ \\
\cmidrule(lr){2-8}
& Exact ratio & $95.35_{\pm0.32}$ & $74.40_{\pm0.53}$ & $93.22_{\pm1.07}$ & $33.33_{\pm0.21}$ & $47.28_{\pm1.09}$ & $68.72_{\pm0.30}$ \\
& KPop$^\dagger$ & $95.40_{\pm0.23}$ & $74.80_{\pm0.53}$ & $94.56_{\pm0.69}$ & $33.21_{\pm0.93}$ & $47.33_{\pm0.39}$ & $69.06_{\pm0.29}$ \\
& IcePop & $95.43_{\pm0.49}$ & $74.93_{\pm0.64}$ & $94.89_{\pm1.17}$ & $32.11_{\pm1.29}$ & $47.58_{\pm0.70}$ & $68.99_{\pm0.51}$ \\
& TIS & $95.15_{\pm0.20}$ & $74.80_{\pm0.72}$ & $94.56_{\pm0.84}$ & $33.58_{\pm1.18}$ & $47.63_{\pm0.45}$ & $69.14_{\pm0.30}$ \\
\cmidrule(lr){2-8}
& Seq-TIS & $95.12_{\pm0.16}$ & $74.67_{\pm0.61}$ & $93.78_{\pm1.02}$ & $33.33_{\pm1.29}$ & $48.07_{\pm0.45}$ & $68.99_{\pm0.61}$ \\
& Seq-MIS & $95.15_{\pm0.35}$ & $73.93_{\pm0.42}$ & $93.89_{\pm1.02}$ & $32.72_{\pm1.10}$ & $47.73_{\pm0.45}$ & $68.68_{\pm0.15}$ \\
& GSPO & $95.55_{\pm0.24}$ & $74.47_{\pm0.50}$ & $94.44_{\pm1.17}$ & $33.21_{\pm0.56}$ & $48.47_{\pm0.73}$ & $69.23_{\pm0.16}$ \\
& FP16 & $94.79_{\pm0.27}$ & $74.13_{\pm0.42}$ & $93.67_{\pm1.45}$ & $33.21_{\pm1.66}$ & $46.83_{\pm0.91}$ & $68.53_{\pm0.31}$ \\
\cmidrule(lr){2-8}
\rowcolor{blue!8}\cellcolor{white}
& CIS (ours) & $\mathbf{95.60}_{\pm0.33}$ & $\mathbf{75.00}_{\pm0.92}$ &
$\mathbf{95.89}_{\pm1.50}$ & $\mathbf{34.31}_{\pm0.56}$ &
$\mathbf{48.62}_{\pm0.45}$ & $\mathbf{69.88}_{\pm0.55}$ \\
\bottomrule
\end{tabularx}
\end{table}

Table~\ref{tab:main} compares all methods under the same training recipe.
CIS achieves the highest five-benchmark average on all three MoE models.
On Qwen1.5-MoE-A2.7B, the average improves from $30.99$ without correction
to $34.78$ with CIS, compared with $32.48$, $33.50$, $34.18$, and $31.40$
for Exact Ratio, KPop$^\dagger$, IcePop, and TIS, respectively.
CIS also achieves $67.88$ on the training-domain GSM8K benchmark.
Its margins over IcePop on several transfer benchmarks are smaller than the
across-seed variation, however, so we do not interpret these individual
columns as statistically separated.

The same trend holds across model families.
CIS reaches an average of $37.16$ on DeepSeek-V2-Lite and $69.88$ on
Qwen3-30B-A3B, the highest mean in both blocks.
For Qwen3-30B-A3B, the base model already reaches $95.22\%$ on GSM8K,
leaving little headroom on the training task; the five-benchmark average
therefore provides the more informative comparison.

Methods that act at other levels show different limitations.
Sequence-level correction accumulates mismatch over the whole response,
while changing numerical precision reduces but does not eliminate the
residual engine discrepancy.
We also evaluate the Qwen1.5-MoE checkpoints on three knowledge and two
code benchmarks that the RL stage never uses (\Cref{tab:ood} in
\Cref{app:experimental-ood}); CIS has the highest pooled accuracy there, but
its margin over IcePop is within the seed spread. Training dynamics and
computational cost are reported in \Cref{app:experimental-dynamics,app:experimental-cost}.

\subsection{RQ2: Why does CIS work, and which design choices matter?}
\label{sec:experiments-ablation}

\begin{figure*}[t]
    \centering
    \includegraphics[width=\textwidth]{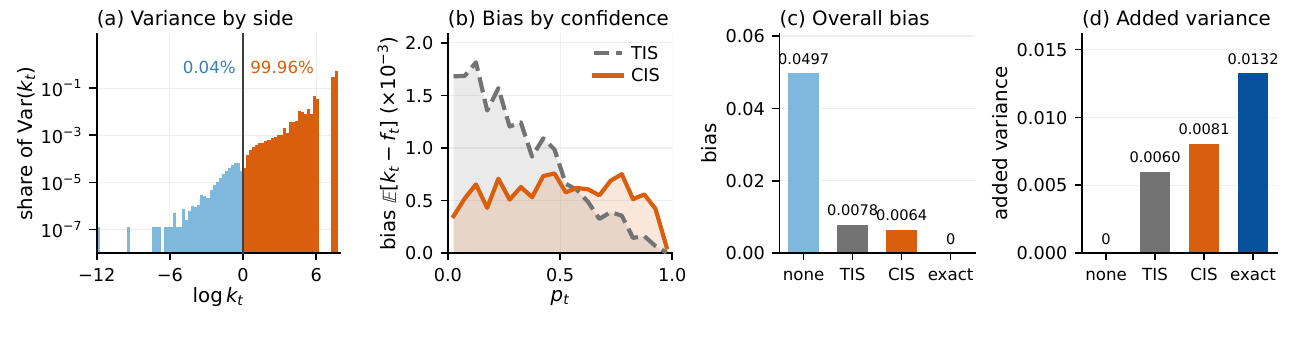}
    \caption{
    \textbf{(a)} Share of $\mathrm{Var}(k_t)$ on each side of $\log k_t=0$.
    \textbf{(b)} Truncation bias across confidence bins.
    \textbf{(c)} Overall bias relative to exact correction.
    \textbf{(d)} Added variance relative to no correction.
    }
    \label{fig:rq2_bv}
\end{figure*}

\begin{table}[htbp!]
\centering
\caption{Ablations of CIS on Qwen1.5-MoE-A2.7B. Held-out accuracy (\%) under the protocol of \Cref{tab:main}; mean $\pm$ standard deviation over three seeds.}
\label{tab:ablation}
\scriptsize\setlength{\tabcolsep}{2pt}
\renewcommand{\arraystretch}{0.95}
\setlength{\aboverulesep}{0.7pt}\setlength{\belowrulesep}{0.7pt}
\begin{tabularx}{\linewidth}{l|*{5}{>{\centering\arraybackslash}X}|>{\centering\arraybackslash}X}
\toprule
Variant & GSM8K & MATH500 & SVAMP & Minerva & Olympiad & Avg \\
\midrule
\rowcolor{blue!8}
CIS ($\lambda=2.3$, $\kappa=5\times10^{-3}$) & $\mathbf{67.88}_{\pm1.64}$ & $\mathbf{24.53}_{\pm1.03}$ & $\mathbf{70.22}_{\pm1.35}$ & $\mathbf{5.88}_{\pm0.64}$ & $\mathbf{5.39}_{\pm0.62}$ & $\mathbf{34.78}_{\pm0.25}$ \\
Fixed ratio cap (TIS, $C=2$) & $61.03_{\pm5.11}$ & $20.73_{\pm3.92}$ & $66.78_{\pm5.09}$ & $4.54_{\pm0.76}$ & $3.91_{\pm0.31}$ & $31.40_{\pm2.47}$ \\
Fixed ratio cap (TIS, $C=1.13$) & $14.99_{\pm17.04}$ & $5.73_{\pm5.06}$ & $20.33_{\pm20.53}$ & $1.47_{\pm1.91}$ & $1.88_{\pm1.86}$ & $8.88_{\pm9.24}$ \\
Two-sided band ($w_t\ge(1+p_t)/2$) & $5.31_{\pm0.33}$ & $4.40_{\pm0.40}$ & $15.67_{\pm1.86}$ & $1.47_{\pm0.64}$ & $3.21_{\pm0.45}$ & $6.01_{\pm0.50}$ \\
Raised floor ($\kappa=0.1$) & $64.42_{\pm1.67}$ & $22.60_{\pm4.06}$ & $64.44_{\pm1.95}$ & $4.90_{\pm0.85}$ & $4.50_{\pm0.37}$ & $32.17_{\pm1.28}$ \\
No floor ($\kappa=0$) & $55.62_{\pm2.69}$ & $20.87_{\pm1.10}$ & $59.78_{\pm2.83}$ & $5.88_{\pm0.37}$ & $2.97_{\pm0.15}$ & $29.02_{\pm1.43}$ \\
\bottomrule
\end{tabularx}
\end{table}

\noindent\textbf{Why truncate only the upper side?}
The variance is almost entirely carried by the upper tail:
on the sampled tokens of the trained MoE at step 86,
tokens with $k_t>1$ account for $99.96\%$ of $\mathrm{Var}(k_t)$,
whereas those with $k_t<1$ contribute only $0.04\%$
(\Cref{fig:rq2_bv}a).
The training ablation shows the same asymmetry:
additionally lifting small displacements to $e^{\varepsilon_t}\ge1/2$,
which gives the two-sided weight
$w_t=\max\{\min\{k_t,1+\lambda\varphi_t\},\,(1+p_t)/2\}$, reduces the
average accuracy from $34.78$ to $6.01$ (see \Cref{tab:ablation}); CIS applies
no lower threshold.
CIS therefore truncates only the upper side and leaves the bounded lower
side unchanged.

\begin{wrapfigure}{r}{0.42\textwidth}
\vspace{-1.2\baselineskip}
\centering
\includegraphics[width=0.42\textwidth]{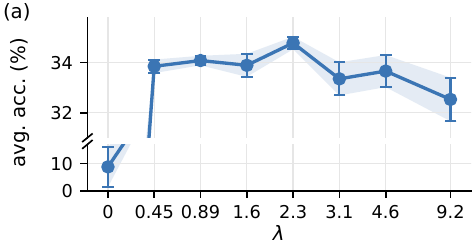}\\[2pt]
\includegraphics[width=0.42\textwidth]{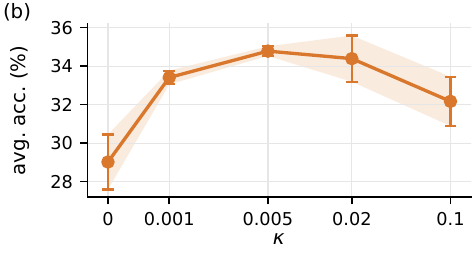}
\vspace{-1.4\baselineskip}
\caption{(a) Threshold $\lambda$ and (b) floor $\kappa$ on Qwen1.5-MoE-A2.7B.}
\label{fig:hyperparameter}
\vspace{-1.2\baselineskip}
\end{wrapfigure}

\noindent\textbf{Why truncate at the source of the mismatch?}
TIS caps the final importance ratio $k_t$, whereas CIS caps the displacement
factor $e^{\varepsilon_t}$, through which the mismatch enters the ratio, at the
common threshold $1+\lambda$ of
\Cref{eq:cis_operator} and then maps this bound back to ratio space.
Because $k_t$ mixes the displacement with token confidence, a fixed cap on
$k_t$ does not represent the same amount of mismatch across tokens.
In contrast, truncating $e^{\varepsilon_t}$ applies the same criterion to the
mismatch itself, while the induced ratio bound automatically adapts to
confidence.

This changes where truncation bias is spent.
\Cref{fig:rq2_bv}b bins the sampled tokens of the base MoE by $p_t$ and
reports $\E[k_t-f_t]$ in each bin.
TIS concentrates its bias on low-confidence tokens: below $p_t=0.3$, the bias
induced by CIS is less than half that of TIS, whereas CIS takes on more bias
at high confidence.
CIS truncates more high-confidence tokens by count, but each such truncation
removes little weight, because the factor $1-p_t$ compresses the ratio on
these tokens; more than half of the weight that CIS removes comes from tokens
with $p_t<0.9$. The two operators differ most in gradient bias at low
confidence (\Cref{fig:rq2_bv}b): for $p_t<0.565$, the cap of CIS is looser
than $C=2$, so CIS removes less weight there than TIS.
\Cref{app:experimental-masking} shows the tokens that each operator acts on
during training.

\Cref{fig:rq2_bv}c,d correspond to the two terms of the mean-squared error in
\Cref{thm:cis_risk}: the bias $\|\E[(f_t-k_t)\ell'_t]\|_2$ and the increase
in the variance of $f_t\ell'_t$ over $f_t\equiv1$.
We compute both with the gradient of $\log p_t$ with respect to the logits in
place of $\ell'_t$, taking expectations over the next-token distributions at
the 6,000 base-model positions of \Cref{app:mm_protocol}; exact correction is
unbiased for the conditional target $G_0$ and therefore has zero bias.
Although CIS truncates more probability mass overall, its overall bias is
lower ($0.0064$ versus $0.0078$ for TIS), while its added variance remains
well below that of Exact Ratio ($0.0081$ versus $0.0132$).
Thus, truncating at the mismatch source changes \emph{where} the clipping
occurs, not only how much is clipped, which yields a more favorable
bias--variance trade-off.

\noindent\textbf{Ablations.}
\Cref{tab:ablation} supports these design choices.
Replacing the confidence-dependent CIS boundary with the fixed TIS cap
reduces the average accuracy from $34.78$ to $31.40$, while the two-sided band
drops it to $6.01$, below the base model ($28.93$).
With $C=1.13$, which truncates a similar share of tokens as CIS ($4.3\%$ versus
$3.8\%$, \Cref{tab:train_stats}), the average falls to $8.88$: the same amount
of truncation harms training when a fixed cap places it on low-confidence tokens.
TIS with $C=1$, which is CIS with $\lambda=0$, collapses as well (\Cref{fig:hyperparameter}a).
The numerical floor is also necessary: removing it lowers the average
accuracy to $29.02$, and raising it to $\kappa=0.1$ lowers it to $32.17$.
Together, the results support the three components of CIS:
upper-side truncation, truncation at the mismatch source, and a small
numerical floor.

\subsection{RQ3: How sensitive is CIS to its hyperparameters?}
\label{sec:experiments-tuning}

\Cref{fig:hyperparameter} varies one hyperparameter of CIS at a time on
Qwen1.5-MoE-A2.7B, with the other fixed at its default; per-benchmark results
are in \Cref{app:experimental-sweeps}.
For the threshold, $\lambda=0$ truncates every ratio above one, which is TIS
with $C=1$, and the average accuracy collapses to $8.75$. Every positive
threshold stays above the uncorrected run ($30.99$), and the default
$\lambda=2.3$ gives the highest accuracy ($34.78$). The other positive
thresholds lie between $32.53$ and $34.08$, so once $\lambda>0$ the average
changes by at most $2.3$ points.
For the floor, accuracy peaks near the storage resolution of
the log-probabilities, where $\kappa=5\times10^{-3}$ and $\kappa=2\times10^{-2}$
reach $34.78$ and $34.39$. A smaller floor of $10^{-3}$ loses $1.4$ points, and
removing the floor lowers the average to $29.02$, below the uncorrected run,
which is consistent with the cap then truncating rounding error on
high-confidence tokens. Raising the floor to $0.1$ lowers the average to
$32.17$; the cap then stays constant on every token with $1-p_t<0.1$ and no
longer tightens with confidence.
\section{Related Work}
\label{sec:related}

RL frameworks separate rollout generation from gradient computation to
improve efficiency~\citep{sheng2025hybridflow,hu2025openrlhf,fu2025areal},
but numerical differences between the engines can yield different token
probabilities even with identical parameters~\citep{yao2025offpolicy}.
In mixture-of-experts models, small perturbations can change discrete
expert selections and amplify the
mismatch~\citep{ma2025r3,ling2025everystep}.
FP16 reduces numerical error~\citep{qi2025fp16}, routing replay aligns
expert selections~\citep{ma2025r3}, and adaptive learning-rate decay
stabilizes optimization~\citep{zhang2026beyondprecision}.
These interventions address different mechanisms: improved precision
does not enforce routing agreement, while replay requires routing traces.
RSPO instead addresses routing drift across policy
updates~\citep{zhang2026rspo}, which differs from engine mismatch at
identical parameters.
CIS complements these approaches by controlling the gradient impact of
residual mismatch without changing numerical precision or recording routes.

Importance sampling corrects distribution mismatch, while truncation
trades bias for reduced
variance~\citep{ionides2008truncated,espeholt2018impala}.
For engine mismatch, TIS caps token ratios~\citep{yao2025offpolicy},
and IcePop masks ratios outside a fixed
interval~\citep{zhao2025icepop,ling2025everystep}.
These thresholds do not account for how token confidence affects the
ratio, whereas KPop addresses this heterogeneity through binary-KL
masking~\citep{guo2026kpop,ling2026lingring}.
Sequence-level alternatives truncate response
weights~\citep{liu2025rlcollapse}, optimize sequence
ratios~\citep{zheng2025gspo}, or reject responses that violate a
trust-region criterion~\citep{li2025trm}.
CISPO also clips weights to preserve token contributions, but uses fixed
bounds on the current-to-old policy ratio~\citep{minimax2025m1}.
CIS instead derives its training--inference ratio cap from a common
threshold on log-odds displacement.
The resulting cap tightens with token confidence, and it keeps a nonzero
weight for truncated tokens.
\section{Conclusion}
\label{sec:conclusion}

We study training--inference mismatch in RLVR and show that it enters the
importance ratio as an additive displacement in log-odds, whose distribution
varies far less with token confidence than the ratio and has a heavy upper tail on
mixture-of-experts models. Based on this characterization, calibrated
importance sampling (CIS) truncates the upper tail of the displacement at a
single threshold, which yields a ratio cap that tightens as token confidence
increases. We prove that CIS bounds the second moment that makes exact
correction unstable, at the cost of a bias controlled by the truncated excess. On three mixture-of-experts models, CIS achieves the highest
five-benchmark average among the evaluated
baselines. Diagnostic analyses show that CIS places less truncation bias on
low-confidence tokens than TIS, and that truncating the
lower side reduces held-out accuracy.

\bibliography{iclr2027_conference}
\bibliographystyle{iclr2027_conference}

\appendix
\appendix

\clearpage
\section*{Appendix}

\startcontents[appendix]
\printcontents[appendix]{}{1}{\setcounter{tocdepth}{2}}

\vspace{1em}

\section{Empirical Anatomy of the Training--Inference Mismatch}
\label{app:mismatch_anatomy}

\Cref{sec:logit_model} derives the ratio--displacement identity
\[
k_t
=
p_t+(1-p_t)e^{\varepsilon_t},
\]
starting from the per-logit perturbation between the training and inference
engines. This appendix provides the measurements behind the empirical claims
used in that derivation and in \Cref{sec:why_correction,sec:optimal_operator}.
We first describe how the per-logit perturbation is measured
(\Cref{app:mm_protocol}). We then compare dense and mixture-of-experts
architectures (\Cref{app:mm_delta_tail}), examine how routing disagreement is
associated with the MoE tail (\Cref{app:mm_routing,app:mm_route_control}), and
verify that the heavy tail propagates from the per-logit perturbation
$\delta_j$ to the log-odds displacement $\varepsilon_t$
(\Cref{app:mm_eps_tail}). Finally, we examine the dependence of
$\varepsilon_t$ on token confidence (\Cref{app:mm_confidence}) and show how
the extreme upper tail dominates the second moment that enters the
importance-sampling variance bound (\Cref{app:mm_moment}).

\subsection{Measurement protocol}
\label{app:mm_protocol}

\paragraph{Models and engine configuration.}
We measure Qwen1.5-MoE-A2.7B-Chat and use Qwen1.5-14B-Chat as a dense
control. All measurements in
\Cref{app:mm_delta_tail,app:mm_routing,app:mm_route_control} use the base
checkpoints, with the training and inference engines loading exactly the same
model parameters. The training engine is HuggingFace Transformers in bf16
with teacher forcing. The inference engine is vLLM 0.26 in bf16 with
eager execution and prefix caching disabled.

\paragraph{Per-logit measurement.}
For each model, we randomly select 250 trajectories from the first shard of
the static mismatch measurement (\Cref{app:experimental-diagnostics}) and sample 24 generation positions from each
trajectory, giving 6,000 positions. At each position, the training engine
records its top-256 pre-softmax logits together with the corresponding
log-normalizer. On the inference side, we truncate the trajectory immediately
before that position and generate one token with \texttt{logprobs=256}, which
returns the next-token distribution at the same prefix. The trajectory and
position identifiers align on all sampled positions. The two top-256 sets
have mean intersection rates of $0.991$ on the dense model and $0.949$ on the
MoE, yielding $1{,}521{,}672$ and $1{,}457{,}477$ paired logit entries,
respectively. The complete measurement takes approximately eight minutes on
four A100 GPUs.

\paragraph{Gauge convention.}
In \Cref{sec:logit_model}, the training-engine logits are written as
$\mathbf z$ and the inference-engine logits as $\mathbf z+\boldsymbol\delta$.
Raw logits are defined only up to an additive constant at each position.
Because the displacement depends only on logit contrasts, adding the same
constant to every coordinate of $\boldsymbol\delta$ does not change
$\varepsilon_t$. We therefore use the normalized log-probability gauge
\begin{equation}
\delta_j
=
\log \pi_{\mathrm{infer}}(j\mid h_t)
-
\log \pi_{\mathrm{train}}(j\mid h_t).
\label{eq:app_delta_gauge}
\end{equation}
This differs from the raw logit perturbation only by a position-specific
constant and is uniquely determined by the two observed distributions. Under
this convention,
\[
\varepsilon_t
=
\operatorname{logit} p_t-\operatorname{logit} q_t,
\]
consistent with the convention used throughout the main text.

For the MoE model, the same measurement additionally records the top-4 expert
identities selected by each engine at every layer, together with the routing
margin between the fourth- and fifth-ranked experts. These records are used in
\Cref{app:mm_routing,app:mm_route_control}.

\subsection{Dense versus MoE: a heavy tail in the per-logit perturbation}
\label{app:mm_delta_tail}

\begin{figure}[!htbp]
    \centering
    \includegraphics[width=0.88\linewidth]{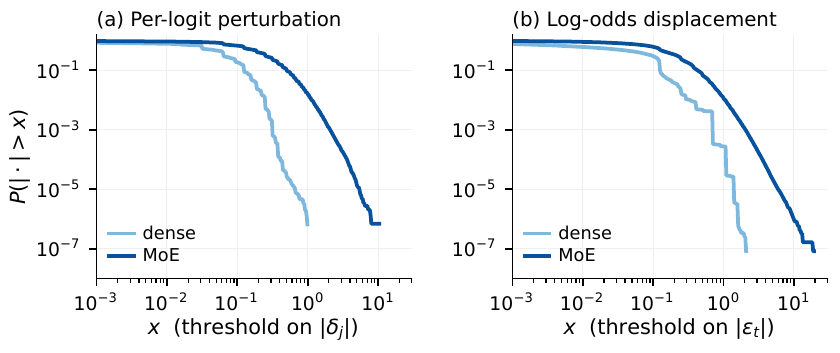}
    \caption{Complementary cumulative distributions of (a) $|\delta_j|$ and (b) $|\varepsilon_t|$ on the dense and MoE models.}
    \label{fig:appA_tails}
\end{figure}

\begin{table}[!htbp]
\centering
\caption{Per-logit perturbation $\delta_j$ on the base checkpoints.}
\label{tab:appA_delta_stats}
\small
\begin{tabular}{lcc}
\toprule
 & Dense & MoE \\
\midrule
Paired entries
    & $1{,}521{,}672$
    & $1{,}457{,}477$ \\
Median
    & $0.0000$
    & $0.0000$ \\
Mean
    & $-0.0013$
    & $-0.0036$ \\
$|\mathrm{mean}|/\mathrm{sd}$
    & $1.6\%$
    & $1.1\%$ \\
$\mathrm{sd}(\delta_j)$
    & $0.079$
    & $0.336$ \\
$q_{99}(|\delta_j|)$
    & $0.250$
    & $1.144$ \\
$q_{99.9}(|\delta_j|)$
    & $0.312$
    & $2.062$ \\
$\Pr(|\delta_j|>0.5)$
    & $0.003\%$
    & $9.87\%$ \\
$\Pr(|\delta_j|>1)$
    & $0.0001\%$
    & $1.52\%$ \\
$\max_j|\delta_j|$
    & $1.00$
    & $10.47$ \\
Range
    & $[-1.00,\,0.91]$
    & $[-7.70,\,10.47]$ \\
\bottomrule
\end{tabular}
\end{table}

\Cref{fig:appA_tails}(a) compares the full tails. The two distributions are
both concentrated near zero, but they separate sharply beyond the central
region. On the dense control, only $0.003\%$ of paired entries satisfy
$|\delta_j|>0.5$, and essentially none exceed one. On the MoE, $9.87\%$
exceed $0.5$, $1.52\%$ exceed one, and the largest perturbation reaches
$10.47$. The standard deviation is correspondingly $0.336$, more than four
times the dense value of $0.079$.

The difference is especially visible in the extreme quantiles:
$q_{99.9}(|\delta_j|)$ grows from $0.312$ on the dense model to $2.062$ on
the MoE. Thus the distinction between the two architectures is not merely a
uniform rescaling of a narrow numerical perturbation; the MoE develops a
pronounced tail that is absent from the dense control.

\subsection{Routing disagreement and the MoE tail}
\label{app:mm_routing}

To examine the source of this additional tail, we use the expert identities
recorded on the MoE. For each sampled position, let $N_{\mathrm{flip}}$
denote the number of layers, out of 24, for which the two engines select
different top-4 expert sets. Routing disagreement is common: the mean is
$3.69$ disagreeing layers per position, and only $8.0\%$ of positions select
the same expert set at all 24 layers.

\begin{figure}[!htbp]
    \centering
    \includegraphics[width=0.84\linewidth]{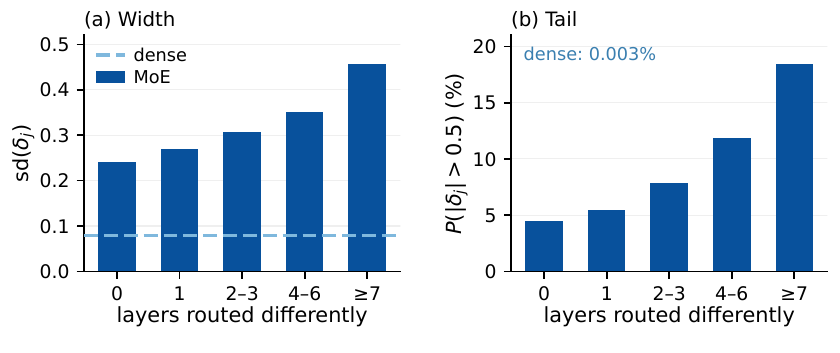}
    \caption{(a) Standard deviation of $\delta_j$ and (b) share of entries with $|\delta_j|>0.5$ against the number of layers with routing disagreement; dashed: dense model.}
    \label{fig:appA_routing}
\end{figure}

\begin{table}[!htbp]
\centering
\caption{Per-logit perturbation against the number of layers $N_{\mathrm{flip}}$ with routing disagreement.}
\label{tab:appA_routing}
\scriptsize
\setlength{\tabcolsep}{3.0pt}
\begin{tabular}{lccccc}
\toprule
$N_{\mathrm{flip}}$
    & $0$ & $1$ & $2$--$3$ & $4$--$6$ & $\ge 7$ \\
\midrule
Positions (\%)
    & $8.0$ & $15.6$ & $33.1$ & $28.2$ & $15.1$ \\
$\mathrm{sd}(\delta_j)$
    & $0.240$ & $0.269$ & $0.307$ & $0.352$ & $0.457$ \\
$\Pr(|\delta_j|>0.5)$ (\%)
    & $4.51$ & $5.42$ & $7.85$ & $11.86$ & $18.48$ \\
$\Pr(|\delta_j|>1)$ (\%)
    & $0.47$ & $0.68$ & $1.11$ & $1.65$ & $3.73$ \\
\bottomrule
\end{tabular}
\end{table}

The two direct perturbation statistics show a clear monotonic trend. From
positions with no routing disagreement to positions with disagreement in at
least seven layers, $\mathrm{sd}(\delta_j)$ increases from $0.240$ to
$0.457$, a $90\%$ increase, while
$\Pr(|\delta_j|>0.5)$ rises from $4.51\%$ to $18.48\%$, a factor of $4.1$.

The conclusion is unchanged under a stricter definition that counts only
routing differences whose gate margin exceeds $10^{-3}$. Under this
``hard-flip'' definition, the conditional standard deviations across the five
bins are
$0.284$, $0.320$, $0.405$, $0.460$, and $0.597$, while
$\Pr(|\delta_j|>0.5)$ increases from
$6.40\%$ to $9.42\%$, $14.72\%$, $20.52\%$, and $30.78\%$.
The association between routing disagreement and the width of the
perturbation is therefore robust to the definition of a route change.

\subsection{Conditional routing control}
\label{app:mm_route_control}

The previous analysis is observational: it stratifies positions by their
realized routing disagreement rather than intervening on the routes.
To separate the architecture-level difference from the additional effect
associated with routing disagreement at the current position, we compare the
dense control with MoE positions whose expert selections either agree at all
24 layers or disagree in at least one layer.

\begin{table}[!htbp]
\centering
\caption{Per-logit perturbation on the dense model and on MoE positions with and without routing disagreement.}
\label{tab:appA_route_control}
\small
\begin{tabular}{lccc}
\toprule
 & $\mathrm{sd}(\delta_j)$
 & $\Pr(|\delta_j|>0.5)$
 & $\Pr(|\delta_j|>1)$ \\
\midrule
Dense
    & $0.079$ & $0.003\%$ & $0.0001\%$ \\
MoE, $N_{\mathrm{flip}}=0$
    & $0.240$ & $4.51\%$ & $0.47\%$ \\
MoE, $N_{\mathrm{flip}}>0$
    & $0.343$ & $10.35\%$ & $1.62\%$ \\
\bottomrule
\end{tabular}
\end{table}

Routing disagreement at the current position therefore does not explain the
entire architecture gap. Even at $N_{\mathrm{flip}}=0$, the MoE standard
deviation is approximately three times that of the dense control.
One possible source is routing disagreement earlier in the prefix, whose
effect propagates through the hidden state to subsequent positions. The
existing measurement cannot isolate this contribution: only $0.25\%$ of
tokens have prefixes for which the two engines agree on every recorded route.

Accordingly, we interpret the results of
\Cref{fig:appA_routing,tab:appA_route_control} as evidence that
\emph{routing disagreement is a major amplifier of the tail, rather than the
sole source of the mismatch}. A causal decomposition would require an
intervention that explicitly replays identical routes across the two engines,
which is outside the scope of the measurements reported here.

\subsection{The heavy tail propagates from $\delta$ to $\varepsilon$}
\label{app:mm_eps_tail}

The identity in \Cref{sec:logit_model} maps the full perturbation vector
$\boldsymbol\delta$ into the sampled token's log-odds displacement,
\[
\varepsilon_t
=
\operatorname{logit}p_t-\operatorname{logit}q_t.
\]
We next verify that the architectural difference observed directly in
$\boldsymbol\delta$ remains visible after this mapping.

Unlike \Cref{app:mm_delta_tail}, which uses the 6,000-position logit dump,
this analysis uses the full static sampled-token measurement: approximately
$12.4$M dense tokens and $12.3$M MoE tokens on the corresponding base
checkpoints. We compute $\varepsilon_t$ whenever neither $p_t$ nor $q_t$
rounds exactly to one.

\begin{table}[!htbp]
\centering
\caption{Log-odds displacement $\varepsilon_t$ in the static measurement on the base checkpoints.}
\label{tab:appA_epsilon_stats}
\small
\begin{tabular}{lcc}
\toprule
 & Dense & MoE \\
\midrule
Sampled tokens
    & $12{,}389{,}115$
    & $12{,}293{,}115$ \\
Finite $\varepsilon_t$
    & $11{,}934{,}872$
    & $12{,}226{,}092$ \\
Finite fraction
    & $96.3\%$
    & $99.5\%$ \\
$\mathrm{sd}(\varepsilon_t)$
    & $0.0997$
    & $0.2947$ \\
$q_{95}(|\varepsilon_t|)$
    & $0.165$
    & $0.599$ \\
$q_{99}(|\varepsilon_t|)$
    & $0.336$
    & $1.055$ \\
$q_{99.9}(|\varepsilon_t|)$
    & $0.693$
    & $2.044$ \\
$\Pr(|\varepsilon_t|>1)$
    & $0.028\%$
    & $1.13\%$ \\
$\Pr(|\varepsilon_t|>2)$
    & $0$
    & $0.106\%$ \\
$\max_t|\varepsilon_t|$
    & $2.13$
    & $19.38$ \\
\bottomrule
\end{tabular}
\end{table}

As shown in \Cref{fig:appA_tails}(b), the dense displacement decays rapidly,
whereas the MoE retains a long tail. The $99.9$th percentile of
$|\varepsilon_t|$ is $0.693$ on the dense model and $2.044$ on the MoE, and
the maximum increases from $2.13$ to $19.38$. Thus the heavy tail identified
at the level of individual logit perturbations survives the nonlinear mapping
to the log-odds displacement that enters the importance ratio.

The dense CCDF contains several visible steps at high confidence. These arise
from the discrete numerical structure of probabilities close to one and
should not be interpreted as a continuous heavy tail. They do not affect the
qualitative dense--MoE separation.

\subsection{The displacement is approximately invariant to token confidence}
\label{app:mm_confidence}

The central design choice of CIS is to place a common upper threshold in the
displacement coordinate rather than a common threshold on the importance
ratio. This choice is motivated by the empirical observation that the scale of
$\varepsilon_t$ changes only mildly with token confidence, whereas the scale
of the ratio deviation changes by orders of magnitude.

We examine the static base-MoE measurement and retain tokens with
$1-p_t\ge10^{-7}$. \Cref{fig:appA_confidence}(a) plots conditional quantiles
of $\varepsilon_t$ against token uncertainty $1-p_t$.
\Cref{fig:appA_confidence}(b) compares the median absolute deviation of
$\varepsilon_t$ with that of $\log k_t$.

\begin{figure}[!htbp]
    \centering
    \includegraphics[width=0.88\linewidth]{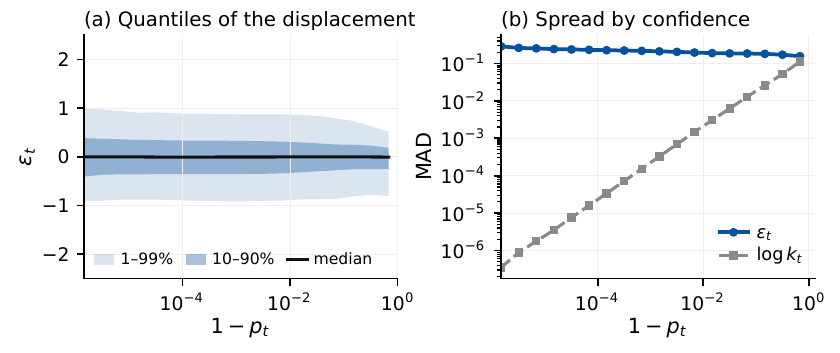}
    \caption{(a) Median, 10--90\% and 1--99\% intervals of $\varepsilon_t$ and (b) median absolute deviation of $\varepsilon_t$ and $\log k_t$ against $1-p_t$ on the base MoE.}
    \label{fig:appA_confidence}
\end{figure}

\begin{table}[!htbp]
\centering
\caption{Conditional statistics of $\varepsilon_t$ and $\log k_t$ on the base MoE.}
\label{tab:appA_confidence}
\scriptsize
\setlength{\tabcolsep}{2.8pt}
\begin{tabular}{lrrrrrr}
\toprule
$p_t$ range
& Tokens
& Med.\ $\varepsilon$
& MAD$(\varepsilon)$
& $q_{05}(\varepsilon)$
& $q_{95}(\varepsilon)$
& MAD$(\log k)$ \\
\midrule
$p_t<0.5$
& $2{,}824{,}021$
& $-0.0059$
& $0.156$
& $-0.375$
& $0.259$
& $1.15{\times}10^{-1}$ \\
$0.5$--$0.9$
& $2{,}337{,}554$
& $-0.0011$
& $0.175$
& $-0.376$
& $0.355$
& $3.96{\times}10^{-2}$ \\
$0.9$--$0.99$
& $1{,}951{,}864$
& $-0.0005$
& $0.186$
& $-0.453$
& $0.420$
& $6.43{\times}10^{-3}$ \\
$0.99$--$0.999$
& $1{,}549{,}915$
& $-0.0019$
& $0.203$
& $-0.492$
& $0.472$
& $6.78{\times}10^{-4}$ \\
$p_t>0.999$
& $3{,}562{,}738$
& $0.0000$
& $0.237$
& $-0.504$
& $0.501$
& $9.90{\times}10^{-6}$ \\
\bottomrule
\end{tabular}
\end{table}

Across 18 logarithmic bins spanning
$1-p_t\approx1.5\times10^{-6}$ to $0.68$, the conditional median of
$\varepsilon_t$ remains within $0.006$ of zero. Its median absolute deviation
ranges from $0.157$ to $0.288$, a factor of only $1.8$. The displacement
distribution is therefore not exactly confidence-invariant, and in
particular becomes moderately wider toward the highest-confidence region,
but the variation is small relative to the ratio coordinate.

Indeed, over the same bins, the median absolute deviation of $\log k_t$
changes by approximately $3.2\times10^5$. This contrast follows the mapping
\[
k_t-1=(1-p_t)(e^{\varepsilon_t}-1):
\]
the same displacement is strongly compressed in ratio space when
$p_t\rightarrow1$. Consistently, on a random subsample of 300,000 tokens the
Spearman correlation between $\varepsilon_t$ and $\log(1-p_t)$ is only
$-0.028$.
Below the floor $\kappa$ of \Cref{alg:cis}, $\varepsilon_t$ is dominated by the
storage resolution of the log-probabilities, so we repeat the comparison on the
$7.6$M tokens with $1-p_t\ge\kappa$, which is the region where CIS acts. There,
the median absolute deviation of $\varepsilon_t$ varies by a factor of $1.3$
across $18$ logarithmic bins, whereas that of $\log k_t$ varies by a factor of
$115$, and the Spearman correlation between $|\varepsilon_t|$ and
$\log(1-p_t)$ is $-0.12$.
The same contrast holds between the SGLang and FSDP engines used for training.
On the tokens with zero staleness in the training logs of the exact-ratio, TIS,
and IcePop runs, the median absolute deviation of $\varepsilon_t$ varies by a
factor of $1.8$ to $2.7$ across confidence bins above the floor, against $150$
to $220$ for $\log k_t$.
These measurements support treating the displacement scale as approximately
common across confidence levels, while making explicit that the approximation
is not exact.

\subsection{The second moment on the trained model}
\label{app:mm_moment}

\Cref{fig:mgf} and \Cref{fig:rq2_bv}a use the MoE after RL training. We load
the step-86 checkpoint of a CIS run on Qwen1.5-MoE-A2.7B into both engines of
\Cref{app:mm_protocol} and score the $510{,}915$ sampled tokens of $4{,}000$
GSM8K rollouts, which are generated by the step-57 snapshot of the same run.
The training engine scores them with teacher forcing and the inference engine
by prefill, so the two engines score the same tokens with the same
parameters. We compute $\varepsilon_t$ whenever neither $p_t$ nor $q_t$ rounds
to one, as in \Cref{app:mm_eps_tail}. The moments of the displacement grow
quickly with the order: $\E[e^{\varepsilon_t}]=7.07$ and
$\E[e^{2\varepsilon_t}]=9.2\times10^{6}$, and the weighted moments
$\E[((1-p_t)e^{\varepsilon_t})^{\alpha}]$ that enter \Cref{thm:exact_risk} are
$0.16$ at $\alpha=1$ and $22.9$ at $\alpha=2$. The second moment is therefore
determined by the upper tail of $\varepsilon_t$, which is the regime in which
the truncated moment of CIS differs from exact correction.

\section{Proofs and Technical Details for Section~\ref{sec:cis}}
\label{app:proofs}

\subsection{Preliminaries and Key Lemmas}
\label{app:lemmas}

In this subsection, we establish the minimal set of identities required for the
proofs of \Cref{thm:exact_risk,thm:cis_risk}. We adopt all notation and settings
from \Cref{sec:cis}. Throughout, $p_t,q_t\in(0,1)$, $k_t:=p_t/q_t$, and the
$n$ token contributions $(k_j,p_j,\ell'_j)_{j=1}^n$ are independent copies of
$(k_t,p_t,\ell'_t)$. For a measurable correction rule
$f:(0,\infty)\times(0,1)\to[0,\infty)$, define
\begin{equation}
\widehat G_f
:=
\frac1n\sum_{j=1}^n f(k_j,p_j)\,\ell'_j,
\qquad
G_0:=\E[k_t\ell'_t].
\label{eq:app_estimator}
\end{equation}
For $\lambda\ge0$, we write
\begin{equation}
f_\lambda(k,p)
:=
\min\big\{k,\;1+\lambda(1-p)\big\}.
\label{eq:f_lambda_def}
\end{equation}

\begin{lemma}[From logit perturbation to log-odds displacement]
\label{lem:logit_to_displacement}
Fix the sampled token $y=y_t$. Let
$p_t=\operatorname{softmax}(z)_y$ and
$q_t=\operatorname{softmax}(z+\delta)_y$, and define
\begin{equation}
w_j
:=
\frac{e^{z_j}}{\sum_{j'\neq y}e^{z_{j'}}},
\qquad j\neq y,
\label{eq:app_w_def}
\end{equation}
and
\begin{equation}
\varepsilon_t
:=
-\delta_y
+
\log\sum_{j\neq y}w_j e^{\delta_j}.
\label{eq:app_eps_def}
\end{equation}
Then
\begin{equation}
k_t
=
\frac{p_t}{q_t}
=
p_t+(1-p_t)e^{\varepsilon_t},
\qquad
k_t-1
=
(1-p_t)\big(e^{\varepsilon_t}-1\big),
\label{eq:app_ratio_displacement}
\end{equation}
and
\begin{equation}
\varepsilon_t
=
\log\frac{p_t}{1-p_t}
-
\log\frac{q_t}{1-q_t}.
\label{eq:app_logodds_identity}
\end{equation}
\end{lemma}

\begin{proof}
Substituting the two softmax probabilities directly into the ratio gives
\begin{align*}
k_t
&=
\frac{
e^{z_y}/\big(e^{z_y}+\sum_{j\neq y}e^{z_j}\big)
}{
e^{z_y+\delta_y}/
\big(e^{z_y+\delta_y}+\sum_{j\neq y}e^{z_j+\delta_j}\big)
}\\
&\overset{(a)}{=}
\frac{e^{z_y}}{e^{z_y}+\sum_{j\neq y}e^{z_j}}
\left(
1+
e^{-\delta_y}
\frac{\sum_{j\neq y}e^{z_j+\delta_j}}{e^{z_y}}
\right)\\
&\overset{(b)}{=}
p_t
+
\frac{
e^{-\delta_y}\sum_{j\neq y}e^{z_j+\delta_j}
}{
e^{z_y}+\sum_{j\neq y}e^{z_j}
}\\
&\overset{(c)}{=}
p_t
+
\frac{\sum_{j\neq y}e^{z_j}}
{e^{z_y}+\sum_{j\neq y}e^{z_j}}
e^{-\delta_y}
\sum_{j\neq y}
\frac{e^{z_j}}
{\sum_{j'\neq y}e^{z_{j'}}}
e^{\delta_j}\\
&\overset{(d)}{=}
p_t
+
(1-p_t)e^{-\delta_y}
\sum_{j\neq y}w_j e^{\delta_j}\\
&\overset{(e)}{=}
p_t+(1-p_t)e^{\varepsilon_t},
\end{align*}
where (a) expands the reciprocal of the inference probability, (b) uses the
definition of $p_t$, (c) multiplies and divides by
$\sum_{j\neq y}e^{z_j}$, (d) uses
\[
1-p_t
=
\frac{\sum_{j\neq y}e^{z_j}}
{e^{z_y}+\sum_{j\neq y}e^{z_j}}
\]
together with \eqref{eq:app_w_def}, and (e) substitutes
\eqref{eq:app_eps_def}. Subtracting
$1=p_t+(1-p_t)$ immediately gives
\begin{align*}
k_t-1
&=
p_t+(1-p_t)e^{\varepsilon_t}
-p_t-(1-p_t)\\
&=
(1-p_t)\big(e^{\varepsilon_t}-1\big).
\end{align*}

To obtain the log-odds form, note that
\begin{align*}
\frac{q_t}{1-q_t}
&=
\frac{e^{z_y+\delta_y}}
{\sum_{j\neq y}e^{z_j+\delta_j}}\\
&\overset{(f)}{=}
\frac{e^{z_y}}
{\sum_{j\neq y}e^{z_j}}
\frac{e^{\delta_y}}
{\sum_{j\neq y}w_j e^{\delta_j}}\\
&\overset{(g)}{=}
\frac{p_t}{1-p_t}
\frac{e^{\delta_y}}
{\sum_{j\neq y}w_j e^{\delta_j}},
\end{align*}
where (f) factors
$\sum_{j\neq y}e^{z_j+\delta_j}$ using \eqref{eq:app_w_def}, and (g) uses
$p_t/(1-p_t)=e^{z_y}/\sum_{j\neq y}e^{z_j}$. Taking logarithms and
rearranging yields
\begin{align*}
\log\frac{p_t}{1-p_t}
-
\log\frac{q_t}{1-q_t}
&=
-\delta_y
+
\log\sum_{j\neq y}w_j e^{\delta_j}\\
&=
\varepsilon_t,
\end{align*}
which proves \eqref{eq:app_logodds_identity}.
\end{proof}

\begin{lemma}[Displacement-coordinate truncation identities]
\label{lem:truncation_identity}
Let $p\in(0,1)$ and suppose $k=p+(1-p)e^{\varepsilon}$. Then
\begin{equation}
\varepsilon\le0
\quad\Longrightarrow\quad
p<k\le1,
\qquad
\varepsilon>0
\quad\Longrightarrow\quad
k>1.
\label{eq:app_one_sided}
\end{equation}
Moreover, for every $\lambda\ge0$,
\begin{equation}
f_\lambda(k,p)
=
p+(1-p)\min\big\{e^{\varepsilon},\,1+\lambda\big\},
\label{eq:app_clipped_ratio}
\end{equation}
and
\begin{equation}
k-f_\lambda(k,p)
=
(1-p)\big(e^{\varepsilon}-1-\lambda\big)_+.
\label{eq:app_clipped_excess}
\end{equation}
Finally, for any fixed ratio cap $C\ge1$,
\begin{equation}
k\le C
\quad\Longleftrightarrow\quad
e^{\varepsilon}
\le
1+\frac{C-1}{1-p}.
\label{eq:app_fixed_cap}
\end{equation}
\end{lemma}

\begin{proof}
If $\varepsilon\le0$, then $e^\varepsilon\le1$, so
\begin{align*}
k
&=
p+(1-p)e^\varepsilon\\
&\le
p+(1-p)
=
1.
\end{align*}
At the same time,
$k-p=(1-p)e^\varepsilon>0$, and hence $p<k\le1$. If
$\varepsilon>0$, then
\[
k-1
=
(1-p)(e^\varepsilon-1)
>
0,
\]
which gives $k>1$.

For the clipped ratio,
\begin{align*}
f_\lambda(k,p)
&\overset{(a)}{=}
\min\Big\{
p+(1-p)e^{\varepsilon},\;
1+\lambda(1-p)
\Big\}\\
&\overset{(b)}{=}
\min\Big\{
p+(1-p)e^{\varepsilon},\;
p+(1-p)(1+\lambda)
\Big\}\\
&\overset{(c)}{=}
p+(1-p)\min\big\{e^{\varepsilon},1+\lambda\big\},
\end{align*}
where (a) substitutes the ratio--displacement identity, (b) rewrites
$1+\lambda(1-p)=p+(1-p)(1+\lambda)$, and (c) uses that
$x\mapsto p+(1-p)x$ is strictly increasing. Therefore
\begin{align*}
k-f_\lambda(k,p)
&=
(1-p)
\left(
e^\varepsilon-\min\{e^\varepsilon,1+\lambda\}
\right)\\
&=
(1-p)(e^\varepsilon-1-\lambda)_+,
\end{align*}
using the pointwise identity
$x-\min\{x,c\}=(x-c)_+$.

For a fixed cap $C\ge1$,
\begin{align*}
k\le C
&\Longleftrightarrow
p+(1-p)e^\varepsilon\le C\\
&\Longleftrightarrow
e^\varepsilon
\le
\frac{C-p}{1-p}\\
&\Longleftrightarrow
e^\varepsilon
\le
1+\frac{C-1}{1-p},
\end{align*}
where the last equality follows from
$C-p=(1-p)+(C-1)$.
\end{proof}

\begin{lemma}[Bias--variance decomposition]
\label{lem:bias_variance}
Let $f:(0,\infty)\times(0,1)\to[0,\infty)$ be measurable with
$\E[f(k_t,p_t)^2\|\ell'_t\|_2^2]<\infty$, and abbreviate
$f_t:=f(k_t,p_t)$. Then
\begin{align}
\E\big\|\widehat G_f-G_0\big\|_2^2
&=
\big\|\E\big[(f_t-k_t)\ell'_t\big]\big\|_2^2
+
\frac1n
\E\big\|
f_t\ell'_t-\E[f_t\ell'_t]
\big\|_2^2
\label{eq:bias_variance_identity}\\
&\le
\big\|\E\big[(f_t-k_t)\ell'_t\big]\big\|_2^2
+
\frac1n
\E\big[f_t^2\|\ell'_t\|_2^2\big].
\label{eq:bias_variance}
\end{align}
\end{lemma}

\begin{proof}
Write $X_j:=f(k_j,p_j)\ell'_j$ and $G_f:=\E[X_j]$. Then
\begin{align*}
\E\big\|\widehat G_f-G_0\big\|_2^2
&\overset{(a)}{=}
\E\big\|
(\widehat G_f-G_f)+(G_f-G_0)
\big\|_2^2\\
&\overset{(b)}{=}
\|G_f-G_0\|_2^2
+
\E\big\|\widehat G_f-G_f\big\|_2^2\\
&\overset{(c)}{=}
\|G_f-G_0\|_2^2
+
\frac1{n^2}
\sum_{i,j=1}^n
\E\big\langle
X_i-G_f,\;
X_j-G_f
\big\rangle\\
&\overset{(d)}{=}
\|G_f-G_0\|_2^2
+
\frac1n
\E\|X_1-G_f\|_2^2\\
&\overset{(e)}{=}
\big\|\E[(f_t-k_t)\ell'_t]\big\|_2^2
+
\frac1n
\E\big\|
f_t\ell'_t-\E[f_t\ell'_t]
\big\|_2^2.
\end{align*}
Here (a) inserts and subtracts $G_f$. The cross term vanishes in (b) because
$\E[\widehat G_f]=G_f$. Step (c) expands
$\widehat G_f-G_f=n^{-1}\sum_j(X_j-G_f)$, while (d) uses independence:
for $i\neq j$,
\[
\E\langle X_i-G_f,X_j-G_f\rangle=0,
\]
and the $n$ diagonal terms are identical. Finally,
\[
G_f-G_0
=
\E[f_t\ell'_t]-\E[k_t\ell'_t]
=
\E[(f_t-k_t)\ell'_t],
\]
which gives (e) and proves \eqref{eq:bias_variance_identity}.

For the variance term,
\begin{align*}
\frac1n\E\|X_1-G_f\|_2^2
&=
\frac1n
\left(
\E\|X_1\|_2^2-\|G_f\|_2^2
\right)\\
&\le
\frac1n\E\|X_1\|_2^2\\
&=
\frac1n
\E\big[f_t^2\|\ell'_t\|_2^2\big],
\end{align*}
where the inequality simply drops the nonpositive term
$-\|G_f\|_2^2/n$. This gives \eqref{eq:bias_variance}.
\end{proof}

\paragraph{Remark (correlated token contributions).}
Independence enters only when the off-diagonal covariance terms are removed.
If the contributions are identically distributed but correlated, define
\begin{equation}
\bar\rho_f
:=
\frac{
\frac{1}{n(n-1)}
\sum_{i\neq j}
\E\langle X_i-G_f,\;X_j-G_f\rangle
}{
\E\|X_1-G_f\|_2^2
}.
\label{eq:app_pairwise_corr}
\end{equation}
If $\bar\rho_f\le\bar\rho$, then
\begin{align*}
\frac1{n^2}
\sum_{i,j=1}^n
\E\big\langle X_i-G_f,\;X_j-G_f\big\rangle
&=
\frac1{n^2}
\left(
n\E\|X_1-G_f\|_2^2
+
\sum_{i\neq j}
\E\langle X_i-G_f,X_j-G_f\rangle
\right)\\
&=
\frac{1+(n-1)\bar\rho_f}{n}
\E\|X_1-G_f\|_2^2\\
&\le
\frac1{n_{\mathrm{eff}}}
\E\|X_1-G_f\|_2^2,
\end{align*}
where
\begin{equation}
n_{\mathrm{eff}}
:=
\frac{n}{1+(n-1)\bar\rho}.
\label{eq:n_eff}
\end{equation}
Hence the variance bounds below remain valid with $n$ replaced by
$n_{\mathrm{eff}}$. Appendix~D reports the corresponding empirical dependence
among the token contributions.

\subsection{Proof of Theorem~\ref{thm:exact_risk}}
\label{app:proof_exact_risk}

We first state the formal version of \Cref{thm:exact_risk}.

\begin{theorem}[Formal version of \Cref{thm:exact_risk}]
\label{thm:exact_risk_formal}
Under the conventions of \Cref{app:lemmas}, assume
$\|\ell'_t\|_2\le b$. Then
\begin{equation}
\E\big\|\widehat G-G_0\big\|_2^2
\le
\frac{b^2}{n}
\E\Big[
p_t^2
+
2p_t(1-p_t)e^{\varepsilon_t}
+
(1-p_t)^2e^{2\varepsilon_t}
\Big],
\label{eq:exact_risk_formal}
\end{equation}
where $\widehat G:=n^{-1}\sum_{j=1}^n k_j\ell'_j$.
\end{theorem}

\begin{proof}[Proof of \Cref{thm:exact_risk}]
By \Cref{lem:logit_to_displacement}, the right-hand side of
\eqref{eq:exact_risk_formal} equals $\frac{b^2}{n}\E[k_t^2]$. If
$\E[k_t^2]=\infty$, there is nothing to prove. Assume therefore that
$\E[k_t^2]<\infty$. Since $\|\ell'_t\|_2\le b$,
\begin{align*}
\E\big[k_t^2\|\ell'_t\|_2^2\big]
&\le
b^2\E\big[k_t^2\big]\\
&
\infty,
\end{align*}
so \Cref{lem:bias_variance} applies with $f_t=k_t$. Because exact correction
has no truncation bias,
\begin{align*}
\E\big\|\widehat G-G_0\big\|_2^2
&=
\frac1n
\E\big\|
k_t\ell'_t-\E[k_t\ell'_t]
\big\|_2^2\\
&\overset{(a)}{\le}
\frac1n
\E\big[k_t^2\|\ell'_t\|_2^2\big]\\
&\overset{(b)}{\le}
\frac{b^2}{n}
\E[k_t^2]\\
&\overset{(c)}{=}
\frac{b^2}{n}
\E\left[
\big(
p_t+(1-p_t)e^{\varepsilon_t}
\big)^2
\right]\\
&=
\frac{b^2}{n}
\E\left[
p_t^2
+
2p_t(1-p_t)e^{\varepsilon_t}
+
(1-p_t)^2e^{2\varepsilon_t}
\right].
\end{align*}
Here (a) drops the nonpositive centered-mean term in the variance, (b) uses
$\|\ell'_t\|_2\le b$, and (c) applies
\Cref{lem:logit_to_displacement}.
\end{proof}

\paragraph{Remark (where the instability enters).}
The proof imposes no structural assumption on the law of $\varepsilon_t$
beyond the moments that appear explicitly in the bound. Exact correction has
zero truncation bias, but its variance depends on the untruncated second-moment
term $\E[(1-p_t)^2e^{2\varepsilon_t}]\le\E[e^{2\varepsilon_t}]$.

\paragraph{Remark (correlated token contributions).}
Under correlated token contributions, the same argument holds with $n$
replaced by $n_{\mathrm{eff}}$ from \eqref{eq:n_eff}. The comparison below
between the untruncated and truncated second moments is otherwise unchanged.

\subsection{Proof of Theorem~\ref{thm:cis_risk}}
\label{app:proof_cis_risk}

We first state the formal version of \Cref{thm:cis_risk}. Throughout this
subsection, let $m_t:=\min\{e^{\varepsilon_t},1+\lambda\}$.

\begin{theorem}[Formal version of \Cref{thm:cis_risk}]
\label{thm:cis_risk_formal}
Under the conventions of \Cref{app:lemmas}, assume
$\|\ell'_t\|_2\le b$. Then for every finite $\lambda\ge0$,
\begin{align}
\E\big\|\widehat G_{f_\lambda}-G_0\big\|_2^2
&\le
b^2
\left(
\E\big[
(1-p_t)(e^{\varepsilon_t}-1-\lambda)_+
\big]
\right)^2
\nonumber\\
&\qquad
+
\frac{b^2}{n}
\E\Big[
\big(
p_t+(1-p_t)\min\{e^{\varepsilon_t},1+\lambda\}
\big)^2
\Big].
\label{eq:cis_risk_formal}
\end{align}
\end{theorem}

\begin{proof}[Proof of \Cref{thm:cis_risk}]
Since
\[
f_\lambda(k_t,p_t)
\le
1+\lambda(1-p_t)
\le
1+\lambda,
\]
we have
$\E[f_\lambda(k_t,p_t)^2\|\ell'_t\|_2^2]
\le b^2(1+\lambda)^2<\infty$, so
\Cref{lem:bias_variance} applies:
\begin{align*}
\E\big\|
\widehat G_{f_\lambda}-G_0
\big\|_2^2
&\le
\big\|
\E[
(f_\lambda(k_t,p_t)-k_t)\ell'_t
]
\big\|_2^2\\
&\qquad
+
\frac1n
\E\big[
f_\lambda(k_t,p_t)^2
\|\ell'_t\|_2^2
\big].
\end{align*}

For the bias term, $f_\lambda(k_t,p_t)\le k_t$, and hence
\begin{align*}
\big\|
\E[
(f_\lambda(k_t,p_t)-k_t)\ell'_t
]
\big\|_2
&\le
\E\left[
\big(k_t-f_\lambda(k_t,p_t)\big)
\|\ell'_t\|_2
\right]\\
&\le
b\,
\E\left[
k_t-f_\lambda(k_t,p_t)
\right]\\
&\overset{(a)}{=}
b\,
\E\left[
(1-p_t)
(e^{\varepsilon_t}-1-\lambda)_+
\right],
\end{align*}
where (a) uses \eqref{eq:app_clipped_excess}.

For the variance term, using \eqref{eq:app_clipped_ratio},
\begin{align*}
\E\big[
f_\lambda(k_t,p_t)^2
\|\ell'_t\|_2^2
\big]
&\le
b^2
\E\big[
f_\lambda(k_t,p_t)^2
\big]\\
&\overset{(b)}{=}
b^2
\E\left[
\big(
p_t+(1-p_t)m_t
\big)^2
\right],
\end{align*}
where (b) substitutes the displacement-coordinate form of the clipped ratio.

Combining the two bounds gives \eqref{eq:cis_risk_formal}.
\end{proof}

\paragraph{Remark (comparison with exact correction).}
The difference between \Cref{thm:exact_risk_formal} and
\Cref{thm:cis_risk_formal} is the replacement of the untruncated displacement
factor $e^{\varepsilon_t}$ by its truncated counterpart $m_t$, together with
the corresponding truncation bias.
Pointwise,
\begin{align*}
\min\big(e^{\varepsilon_t},1+\lambda\big)^2
&\le
e^{2\varepsilon_t},\\
\min\big(e^{\varepsilon_t},1+\lambda\big)^2
&\le
(1+\lambda)^2,
\end{align*}
and therefore
\begin{equation}
\E\big[
\min(e^{\varepsilon_t},1+\lambda)^2
\big]
\le
\min\left\{
\E[e^{2\varepsilon_t}],
\;
(1+\lambda)^2
\right\}.
\label{eq:app_truncated_second_moment}
\end{equation}
Moreover, $p_t+(1-p_t)m_t\le\min\{k_t,\,1+\lambda\}$, so the variance term of
\eqref{eq:cis_risk_formal} satisfies
$\E[(p_t+(1-p_t)m_t)^2]\le\min\{\E[k_t^2],\,(1+\lambda)^2\}$.
Thus, regardless of how heavy the positive tail of $\varepsilon_t$ becomes,
the second-moment term entering the CIS variance bound is at most
$(1+\lambda)^2$.

The corresponding bias remains controlled by the first moment of the truncated
excess, since
\begin{align*}
(1-p_t)(e^{\varepsilon_t}-1-\lambda)_+
&\le
(e^{\varepsilon_t}-1-\lambda)_+\\
&\le
e^{\varepsilon_t}
\mathbf 1\{
e^{\varepsilon_t}>1+\lambda
\}\\
&\le
e^{\varepsilon_t}.
\end{align*}
CIS therefore trades a finite upper-tail bias for a uniformly bounded
second-moment contribution, which is the bias--variance trade-off described in
Section~3.3.

\paragraph{Remark (correlated token contributions).}
The same result holds with $n$ replaced by $n_{\mathrm{eff}}$ under the
correlated-token setting above. The truncation term and the bound
\eqref{eq:app_truncated_second_moment} are pointwise and therefore do not depend
on the effective sample size.
\subsection{Numerical Floor in Algorithm~1}
\label{app:numerical_floor}

The theoretical operator in \eqref{eq:f_lambda_def} uses the uncertainty
$1-p_t$. Algorithm~1 replaces it by
$\phi_\kappa(p):=\max\{1-p,\kappa\}$, where $0<\kappa\le1$, and applies
\begin{equation}
f_{\lambda,\kappa}(k,p)
:=
\min\big\{
k,\;
1+\lambda\phi_\kappa(p)
\big\}.
\label{eq:app_floor_operator}
\end{equation}
This modification prevents the confidence-dependent cap from becoming narrower
than the numerical resolution of the stored log-probabilities.

\begin{corollary}[Effect of the numerical floor]
\label{cor:floor_bound}
For every $\lambda\ge0$ and $\kappa\in(0,1]$,
\begin{equation}
f_\lambda(k,p)
\le
f_{\lambda,\kappa}(k,p)
\le
k,
\label{eq:app_floor_order}
\end{equation}
and consequently
\begin{equation}
0
\le
k-f_{\lambda,\kappa}(k,p)
\le
k-f_\lambda(k,p).
\label{eq:app_floor_bias_order}
\end{equation}
Moreover,
\begin{equation}
f_{\lambda,\kappa}(k,p)
\le
1+\lambda.
\label{eq:app_floor_upper_bound}
\end{equation}
Hence, if $\|\ell'_t\|_2\le b$,
\begin{equation}
\E\big[
f_{\lambda,\kappa}(k_t,p_t)^2
\|\ell'_t\|_2^2
\big]
\le
b^2(1+\lambda)^2.
\label{eq:app_floor_second_moment}
\end{equation}
\end{corollary}

\begin{proof}
By definition,
$\phi_\kappa(p)=\max\{1-p,\kappa\}\ge1-p$. Since $\lambda\ge0$,
\[
1+\lambda\phi_\kappa(p)
\ge
1+\lambda(1-p).
\]
The map $c\mapsto\min\{k,c\}$ is nondecreasing, so
\begin{align*}
f_{\lambda,\kappa}(k,p)
&=
\min\{k,1+\lambda\phi_\kappa(p)\}\\
&\ge
\min\{k,1+\lambda(1-p)\}\\
&=
f_\lambda(k,p).
\end{align*}
The definition of the minimum also gives
$f_{\lambda,\kappa}(k,p)\le k$, which proves
\eqref{eq:app_floor_order}. Subtracting both sides from $k$ yields
\eqref{eq:app_floor_bias_order}.

Since both $1-p$ and $\kappa$ lie in $(0,1]$,
$\phi_\kappa(p)\le1$. Therefore
\begin{align*}
f_{\lambda,\kappa}(k,p)
&\le
1+\lambda\phi_\kappa(p)\\
&\le
1+\lambda,
\end{align*}
which gives \eqref{eq:app_floor_upper_bound}. If
$\|\ell'_t\|_2\le b$, it follows immediately that
\begin{align*}
\E\big[
f_{\lambda,\kappa}(k_t,p_t)^2
\|\ell'_t\|_2^2
\big]
&\le
b^2
\E\big[
f_{\lambda,\kappa}(k_t,p_t)^2
\big]\\
&\le
b^2(1+\lambda)^2,
\end{align*}
proving \eqref{eq:app_floor_second_moment}.
\end{proof}

\paragraph{Remark.}
The floor raises the cap only at the extreme high-confidence end, so
$f_{\lambda,\kappa}$ removes no more of the exact importance weight than the
idealized operator $f_\lambda$. At the same time,
\eqref{eq:app_floor_upper_bound} keeps the applied weight uniformly bounded by
$1+\lambda$. The numerical safeguard in Algorithm~1 therefore preserves the
central variance-control property of CIS while preventing the theoretical band
from becoming narrower than the storage resolution.
\section{Experimental Details}
\label{app:experimental}

\subsection{Implementation and Hardware}
\label{app:experimental-implementation}

\noindent\textbf{Recipe.}
Every arm uses the same recipe: learning rate $3\times10^{-6}$, weight decay $0.01$, PPO clip range $(0.2,0.28)$, no KL penalty, four samples per prompt at temperature one, at most 1{,}024 generated tokens, at most two versions of rollout staleness, three epochs over the 7{,}473 GSM8K training problems (87 optimizer steps), and one node with eight A100-80GB GPUs per arm, four serving SGLang and four training with FSDP and flash attention.

\noindent\textbf{Operators.}
Training uses AReaL v2.0.0 with its decoupled loss, in which the rollout-side and training-side log-probabilities are both available per token. All operators are implemented in the same token-level branch that AReaL uses for TIS and IcePop, so that every arm differs only in the weight on the token loss; CIS multiplies the loss by the weight $\min\{k_t,1+\lambda\varphi_t\}$ of \Cref{alg:cis}, with $\varphi_t=\max(1-p_t,\kappa)$ and $p_t$ taken from the training-side log-probability, and the weight carries no gradient. The two-sided band of Table~\ref{tab:ablation} additionally lifts the displacement from below to $e^{\varepsilon_t}\ge1/2$, which maps through \Cref{eq:k_geometry} to the weight $\max\{\min\{k_t,1+\lambda\varphi_t\},\,(1+p_t)/2\}$; CIS applies no lower threshold. Table~\ref{tab:baselines} lists the operator and the hyperparameters of every arm. Beyond the operators, the changes to the framework are longer worker-readiness timeouts, a numerically corrected binary-KL metric for KPop$^\dagger$, product weights over a response for Seq-TIS and Seq-MIS, static loss scaling for the fp16 arm, and a token-level dump hook. The training engine must use flash attention; with packed sequences the scaled-dot-product fallback attends across sequence boundaries and corrupts the training-side log-probabilities, which we detected as a median $\log k_t$ far from zero on every method. Weights are sent from the trainer to the rollout engine by collective communication after every optimizer step. On Qwen1.5-MoE each arm takes about 1.0 to 1.4 hours for the 87 steps, and Appendix~\ref{app:experimental-cost} breaks this time down by phase. DeepSeek-V2-Lite \citep{deepseekai2024deepseekv2} and Qwen3-30B-A3B \citep{yang2025qwen3} use the same recipe with vLLM rollouts, weights reloaded from disk after every step, and optimizer states in bf16; the thinking mode of Qwen3-30B-A3B is disabled in both training and evaluation.

\begin{table}[!tb]
\centering
\caption{Operators and hyperparameters of the arms in Table~\ref{tab:main}. $K=\prod_t k_t$ is the product over a response and $\varphi_t=\max(1-p_t,\kappa)$.}
\label{tab:baselines}
\footnotesize\setlength{\tabcolsep}{4pt}
\renewcommand{\arraystretch}{1.05}
\begin{tabular}{lll}
\toprule
Method & Weight on the token loss & Hyperparameters \\
\midrule
No correction & $1$ (decoupled loss disabled) & --- \\
Exact ratio & $k_t$ & clamp to $[10^{-6},10^{6}]$ \\
TIS & $\min(k_t,C)$ & $C=2$ \\
IcePop & $k_t\,\1\{k_t\in[\alpha,\beta]\}$ & $[\alpha,\beta]=[0.5,5]$ \\
KPop$^\dagger$ & $k_t\,\1\{\max(\mathrm{KL}_{\mathrm{B}}(p\|q),\mathrm{KL}_{\mathrm{B}}(q\|p))\le\phi\}$, float64 & $\phi=2$ \\
\midrule
Seq-TIS & $\min(K,C)$ for every token of the response & $C=2$ \\
Seq-MIS & $K\,\1\{K\in[\alpha,\beta]\}$ for every token of the response & $[\alpha,\beta]=[0.5,2]$ \\
GSPO & sequence-level surrogate, $k_t$ uncorrected & clip $(3\times10^{-4},4\times10^{-4})$ \\
FP16 & $1$, rollout and training in fp16 & loss scale $1024$ \\
\midrule
CIS (ours) & $\min\{k_t,1+\lambda\varphi_t\}$ & $\lambda=2.3$, $\kappa=5\times10^{-3}$ \\
\bottomrule
\end{tabular}
\end{table}

\subsection{Datasets and Evaluation Protocol}
\label{app:experimental-protocol}

\noindent\textbf{Training data.}
Every RL run trains on the 7{,}473 problems of the GSM8K training split \citep{cobbe2021gsm8k}. The prompt is the problem followed by an instruction to put the final answer in \verb|\boxed{}| (Appendix~\ref{app:experimental-prompts}), and the reward is one when math-verify judges the final answer equal to the reference answer and zero otherwise. The framework evaluates the GSM8K test split once per epoch for logging; these numbers are not used to select a checkpoint, and every number that we report comes from the final checkpoint.

\noindent\textbf{Measurement data.}
The static measurement of the mismatch uses 2{,}500 problems drawn at random from DAPO-Math-17k \citep{yu2025dapo}. This set is not used for training.

\noindent\textbf{Protocols.}
Table~\ref{tab:benchmarks} lists the evaluation benchmarks. Tables~\ref{tab:main} and~\ref{tab:ablation} use the chat template of each model with the instruction to put the final answer in \verb|\boxed{}|, greedy decoding with vLLM, at most 1{,}536 new tokens, and the math-verify judge. The out-of-domain evaluation of Table~\ref{tab:ood} uses the same template and greedy decoding; a multiple-choice item is scored by the option letter that the response states, with at most 512 new tokens, and a code item is scored by running the completion against the benchmark tests with a limit of 15 seconds, with at most 768 new tokens. Changing the tensor-parallel degree of the evaluation engine changes a few greedy decisions (two HumanEval and six MBPP items on one checkpoint), so every checkpoint of a table is evaluated with the same degree.

\begin{table}[H]
\centering
\caption{Evaluation benchmarks.}
\label{tab:benchmarks}
\footnotesize\setlength{\tabcolsep}{4pt}
\begin{tabularx}{\linewidth}{@{}>{\raggedright\arraybackslash}X l
  >{\raggedright\arraybackslash}p{3.1cm} l r@{}}
\toprule
Benchmark & Domain & Split & Scoring & Items \\
\midrule
GSM8K \citep{cobbe2021gsm8k} & math & test & math-verify & 1{,}319 \\
MATH-500 \citep{hendrycks2021math,lightman2024verify} & math & test subset & math-verify & 500 \\
SVAMP \citep{patel2021svamp} & math & test & math-verify & 300 \\
Minerva-Math \citep{lewkowycz2022minerva} & math & test & math-verify & 272 \\
OlympiadBench \citep{he2024olympiadbench} & math & English, text-only, open-ended & math-verify & 674 \\
\midrule
MMLU-STEM \citep{hendrycks2021mmlu} & knowledge & test, 19 STEM subjects & option letter & 3{,}153 \\
ARC-Challenge \citep{clark2018arc} & knowledge & test & option letter & 1{,}172 \\
OpenBookQA \citep{mihaylov2018openbookqa} & knowledge & test & option letter & 500 \\
HumanEval \citep{chen2021humaneval} & code & test & unit tests & 164 \\
MBPP \citep{austin2021mbpp} & code & test & unit tests & 500 \\
\bottomrule
\end{tabularx}
\end{table}

\subsection{Prompts}
\label{app:experimental-prompts}

Every prompt below is the user message, formatted with the chat template of the model and no additional system message; the thinking mode of Qwen3-30B-A3B is disabled. Placeholders in braces are filled with the fields of each item.

\begin{prompt}[title={Prompt: RL training on GSM8K}]
\texttt{\{question\}}\\
Please put your final answer within \verb|\boxed{}|.
\end{prompt}

\begin{prompt}[title={Prompt: evaluation on the five math benchmarks (Table~\ref{tab:main})}]
\texttt{\{question\}}\\[4pt]
Please reason step by step, and put your final answer within \verb|\boxed{}|.
\end{prompt}

\begin{prompt}[title={Prompt: out-of-domain multiple choice (MMLU-STEM, ARC-Challenge, OpenBookQA)}]
\texttt{\{question\}}\\[4pt]
A. \texttt{\{option A\}}\\
B. \texttt{\{option B\}}\\
C. \texttt{\{option C\}}\\
D. \texttt{\{option D\}}\\[4pt]
Answer with the letter of the correct option. Put the single letter within \verb|\boxed{}|.
\end{prompt}

\begin{prompt}[title={Prompt: out-of-domain code (HumanEval)}]
Complete the following Python function.\\[4pt]
\verb|```python|\\
\texttt{\{function signature and docstring\}}\\
\verb|```|\\[4pt]
Return the complete function in a single \verb|```python| code block. Do not include example usage or tests.
\end{prompt}

\begin{prompt}[title={Prompt: out-of-domain code (MBPP)}]
\texttt{\{problem description\}}\\[4pt]
Your code must pass these tests:\\
\verb|```python|\\
\texttt{\{three assert statements\}}\\
\verb|```|\\[4pt]
Return the complete function in a single \verb|```python| code block. Do not include example usage or tests.
\end{prompt}

\begin{prompt}[title={Prompt: measurement of the mismatch (DAPO-Math-17k)}]
\texttt{\{problem\}}
\end{prompt}

\subsection{Diagnostics}
\label{app:experimental-diagnostics}\label{app:diagnostics}

\noindent\textbf{Static measurement.}
Prompts are 2{,}500 problems drawn at random from DAPO-Math-17k. Generation uses vLLM 0.26 in bf16 with eager execution and without prefix caching; each prompt receives eight samples at temperature one and top-$p$ one, at most 1{,}024 tokens per response, and the log-probability of every sampled token is taken from the sampling step itself rather than from a later prefill pass. Rescoring uses HuggingFace Transformers in bf16 with teacher forcing on the exact sampled sequence and the same parameter snapshot; the token at position $t$ is scored from the logits at position $t-1$ on both sides. For the MoE model we record the top-4 expert identities of every layer on both sides and the gate margin between the fourth and fifth expert; the routing read from forward hooks agrees with the identities returned by the fused kernel on $100\%$ of the (token, layer) pairs whose margin is positive. Two controls fix the measurement floor: rescoring a sequence twice on the training engine returns exactly zero difference, and the empirical mean of $k_t$ is $1.0000$ on both architectures, which verifies the per-step calibration identity $\E[k_t\mid h_t]=1$. Two further controls accompany the main measurement: rescoring in fp32, and rescoring by a prefill pass on the inference engine, which reproduces the decode-time bulk. The displacement is computed as $\varepsilon_t=\operatorname{logit}p_t-\operatorname{logit}q_t$ from the two log-probabilities.

\noindent\textbf{Training-time logging.}
The training runs dump, every ten optimizer calls, the rollout-side log-probability, the training-side log-probability at the start of the update, and the parameter version of the rollout for a sample of tokens. Version staleness is the difference between the current version and the rollout version.

\subsection{Preregistered Evaluation Contract}
\label{app:experimental-contract}

Held-out accuracy at the final checkpoint is the arbiter of every comparison; training reward is reported but is not used to choose between methods. The following choices were fixed before the corresponding runs. The floor is $\kappa=5\times10^{-3}$, set from the storage resolution of the log-probabilities before any training run. The slope grid is $\lambda\in\{1.6,2.3,3.1\}$. All arms use the recipe of Section~\ref{sec:experiments-setup} without per-arm tuning. Every ablation changes one element of the operator relative to the deployed configuration. The main comparison and the ablations of Table~\ref{tab:ablation} use three seeds.

\section{Additional Results and Limitations}
\label{app:discussion}

\subsection{Training Dynamics}
\label{app:experimental-dynamics}

\begin{figure}[!htbp]
\centering
\includegraphics[width=\linewidth]{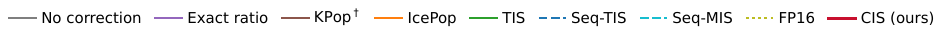}\\[3pt]
\begin{minipage}[t]{0.325\linewidth}\centering
\includegraphics[width=\linewidth]{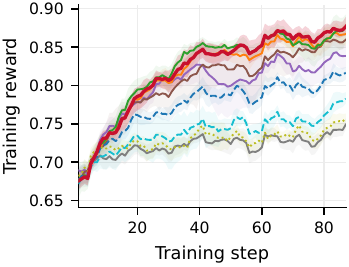}\\[-2pt]{\small (a)}
\end{minipage}\hfill
\begin{minipage}[t]{0.325\linewidth}\centering
\includegraphics[width=\linewidth]{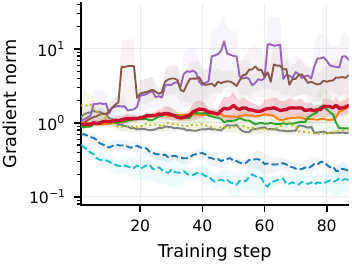}\\[-2pt]{\small (b)}
\end{minipage}\hfill
\begin{minipage}[t]{0.325\linewidth}\centering
\includegraphics[width=\linewidth]{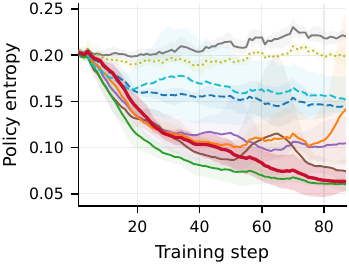}\\[-2pt]{\small (c)}
\end{minipage}\\[4pt]
\begin{minipage}[t]{0.325\linewidth}\centering
\includegraphics[width=\linewidth]{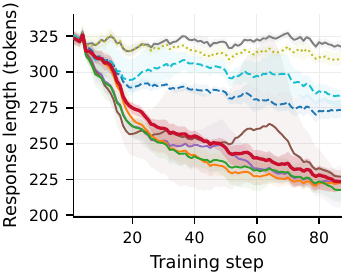}\\[-2pt]{\small (d)}
\end{minipage}\hfill
\begin{minipage}[t]{0.325\linewidth}\centering
\includegraphics[width=\linewidth]{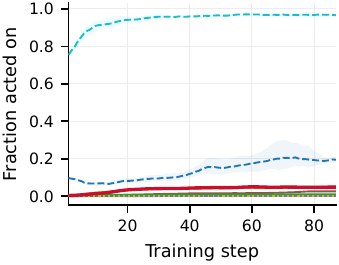}\\[-2pt]{\small (e)}
\end{minipage}\hfill
\begin{minipage}[t]{0.325\linewidth}\centering
\includegraphics[width=\linewidth]{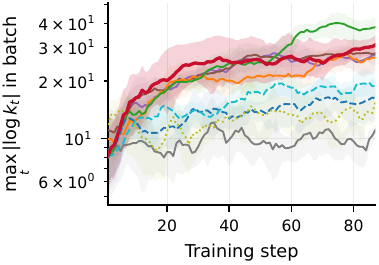}\\[-2pt]{\small (f)}
\end{minipage}
\caption{Training dynamics on Qwen1.5-MoE-A2.7B (mean over seeds, band spanning the seeds, five-step smoothing). (a) Reward. (b) Gradient norm. (c) Entropy. (d) Response length. (e) Fraction of tokens acted on, counting whole responses for Seq-TIS and Seq-MIS. (f) Largest $|\log k_t|$ in the batch.}
\label{fig:dynamics}
\end{figure}

\begin{table}[!htbp]
\centering
\caption{Training statistics on Qwen1.5-MoE-A2.7B, mean $\pm$ standard deviation over three seeds: third-epoch mean (median for the gradient norm and $\max_t|\log k_t|$), and peak gradient norm and fraction acted on over the whole run.}
\label{tab:train_stats}
\scriptsize\setlength{\tabcolsep}{2pt}
\renewcommand{\arraystretch}{0.95}
\setlength{\aboverulesep}{0.7pt}\setlength{\belowrulesep}{0.7pt}
\begin{tabularx}{\linewidth}{l|*{7}{>{\centering\arraybackslash}X}}
\toprule
Method & Reward & Grad.\ norm & Peak grad. & Entropy & Length & Acted on (\%) & $\max_t|\log k_t|$ \\
\midrule
No correction & $0.732_{\pm0.004}$ & $0.75_{\pm0.02}$ & $2.5_{\pm0.4}$ & $0.220_{\pm0.005}$ & $322_{\pm2}$ & $0.00_{\pm0.00}$ & $9.5_{\pm0.5}$ \\
Exact ratio & $0.829_{\pm0.012}$ & $3.58_{\pm0.75}$ & $69.5_{\pm41.4}$ & $0.102_{\pm0.013}$ & $227_{\pm4}$ & $0.06_{\pm0.03}$ & $25.9_{\pm3.4}$ \\
KPop$^\dagger$ & $0.847_{\pm0.016}$ & $3.09_{\pm0.76}$ & $29.2_{\pm23.5}$ & $0.093_{\pm0.033}$ & $245_{\pm29}$ & $1.33_{\pm0.05}$ & $27.3_{\pm1.7}$ \\
IcePop & $0.859_{\pm0.005}$ & $1.12_{\pm0.04}$ & $2.8_{\pm1.5}$ & $0.113_{\pm0.035}$ & $224_{\pm7}$ & $3.64_{\pm0.40}$ & $24.0_{\pm2.3}$ \\
TIS ($C=2$) & $0.861_{\pm0.014}$ & $0.97_{\pm0.11}$ & $3.6_{\pm1.6}$ & $0.064_{\pm0.004}$ & $226_{\pm9}$ & $0.76_{\pm0.07}$ & $37.7_{\pm4.4}$ \\
TIS ($C=1.13$) & $0.780_{\pm0.065}$ & $0.91_{\pm0.08}$ & $8.5_{\pm11.7}$ & $0.200_{\pm0.142}$ & $224_{\pm23}$ & $4.32_{\pm0.64}$ & $29.7_{\pm4.4}$ \\
\midrule
Seq-TIS & $0.803_{\pm0.004}$ & $0.25_{\pm0.02}$ & $0.8_{\pm0.0}$ & $0.147_{\pm0.016}$ & $276_{\pm7}$ & $13.54_{\pm2.83}$ & $15.6_{\pm2.9}$ \\
Seq-MIS & $0.759_{\pm0.002}$ & $0.15_{\pm0.01}$ & $0.6_{\pm0.2}$ & $0.158_{\pm0.025}$ & $292_{\pm9}$ & $94.56_{\pm0.50}$ & $18.2_{\pm1.7}$ \\
FP16 & $0.740_{\pm0.002}$ & $0.74_{\pm0.01}$ & $3.2_{\pm1.3}$ & $0.202_{\pm0.002}$ & $313_{\pm3}$ & $0.00_{\pm0.00}$ & $12.7_{\pm2.1}$ \\
\midrule
\rowcolor{blue!8}
CIS (ours) & $\mathbf{0.867}_{\pm0.009}$ & $1.47_{\pm0.31}$ & $3.0_{\pm1.1}$ & $0.071_{\pm0.017}$ & $232_{\pm8}$ & $3.81_{\pm0.42}$ & $26.9_{\pm5.1}$ \\
\bottomrule
\end{tabularx}
\end{table}

Figure~\ref{fig:dynamics} and Table~\ref{tab:train_stats} show how the arms of Table~\ref{tab:main} train on Qwen1.5-MoE. The token-level operators that bound the ratio (TIS, IcePop, and CIS) reach a training reward between $0.859$ and $0.867$ with a peak gradient norm below four, whereas the exact ratio and KPop$^\dagger$ reach a lower reward with peak gradient norms of $70$ and $29$. The uncorrected run and the fp16 run stay near a reward of $0.74$, and their entropy and response length stay close to their initial values. Seq-MIS removes $95\%$ of the responses, and its gradient norm is the smallest of all arms because most of the batch carries no gradient. The largest $|\log k_t|$ in the batch stays near $10$ for the uncorrected and fp16 runs and grows to between $16$ and $38$ for the other arms as the policy sharpens. The fraction of tokens that CIS acts on rises during the first third of training and then stays between $4\%$ and $5\%$ (panel e).

\subsection{Hyperparameter Sweeps}
\label{app:experimental-sweeps}

Tables~\ref{tab:sweep_lambda} and~\ref{tab:sweep_kappa} list the per-benchmark accuracy behind \Cref{fig:hyperparameter}. Each sweep varies one hyperparameter of CIS on Qwen1.5-MoE-A2.7B with the other fixed at its default, and uses the recipe and evaluation protocol of \Cref{tab:main}; the default rows are the CIS entries of \Cref{tab:main}. Setting $\lambda=0$ truncates every ratio above one, which is TIS with $C=1$, and the average accuracy collapses to $8.75$ with a large spread across seeds. Every positive threshold stays above the uncorrected run ($30.99$), and the default $\lambda=2.3$ gives the highest average ($34.78$), while the other positive thresholds lie between $32.53$ and $34.08$.

\begin{table}[H]
\centering
\caption{Accuracy (\%) of CIS on Qwen1.5-MoE-A2.7B for different thresholds $\lambda$ with $\kappa=5\times10^{-3}$. Mean $\pm$ standard deviation over three seeds.}
\label{tab:sweep_lambda}
\scriptsize\setlength{\tabcolsep}{2pt}
\renewcommand{\arraystretch}{0.95}
\setlength{\aboverulesep}{0.7pt}\setlength{\belowrulesep}{0.7pt}
\begin{tabularx}{\linewidth}{l|*{5}{>{\centering\arraybackslash}X}|>{\centering\arraybackslash}X}
\toprule
$\lambda$ & GSM8K & MATH500 & SVAMP & Minerva & Olympiad & Avg \\
\midrule
$0$ & $14.71_{\pm13.47}$ & $6.80_{\pm5.67}$ & $19.33_{\pm17.17}$ & $0.98_{\pm1.39}$ & $1.93_{\pm1.12}$ & $8.75_{\pm7.44}$ \\
$0.45$ & $64.82_{\pm1.88}$ & $23.33_{\pm2.44}$ & $70.33_{\pm2.08}$ & $7.11_{\pm0.21}$ & $3.66_{\pm0.17}$ & $33.85_{\pm0.27}$ \\
$0.89$ & $66.59_{\pm1.95}$ & $22.87_{\pm1.75}$ & $70.22_{\pm1.90}$ & $6.25_{\pm1.33}$ & $4.45_{\pm0.83}$ & $34.08_{\pm0.19}$ \\
$1.6$ & $65.68_{\pm0.46}$ & $23.33_{\pm1.63}$ & $68.78_{\pm2.17}$ & $6.74_{\pm0.56}$ & $4.90_{\pm0.59}$ & $33.89_{\pm0.46}$ \\
\rowcolor{blue!8}
$2.3$ (default) & $67.88_{\pm1.64}$ & $24.53_{\pm1.03}$ & $70.22_{\pm1.35}$ & $5.88_{\pm0.64}$ & $5.39_{\pm0.62}$ & $34.78_{\pm0.25}$ \\
$3.1$ & $65.07_{\pm0.70}$ & $23.67_{\pm0.92}$ & $67.89_{\pm2.50}$ & $6.13_{\pm1.49}$ & $4.01_{\pm0.90}$ & $33.35_{\pm0.66}$ \\
$4.6$ & $64.54_{\pm1.19}$ & $24.20_{\pm0.53}$ & $69.22_{\pm1.26}$ & $6.86_{\pm0.21}$ & $3.46_{\pm0.09}$ & $33.66_{\pm0.64}$ \\
$9.2$ & $62.27_{\pm1.67}$ & $23.40_{\pm0.60}$ & $66.89_{\pm1.84}$ & $6.74_{\pm0.21}$ & $3.36_{\pm0.09}$ & $32.53_{\pm0.87}$ \\
\bottomrule
\end{tabularx}
\end{table}

For the floor, the accuracy peaks near the storage resolution of the log-probabilities, where $\kappa=5\times10^{-3}$ and $\kappa=2\times10^{-2}$ reach $34.78$ and $34.39$; a smaller floor of $10^{-3}$ loses $1.4$ points, and removing the floor or raising it to $0.1$ lowers the average to $29.02$ and $32.17$. The gains of the default floor are largest on GSM8K and OlympiadBench.

\begin{table}[H]
\centering
\caption{Accuracy (\%) of CIS on Qwen1.5-MoE-A2.7B for different floors $\kappa$ with $\lambda=2.3$. Mean $\pm$ standard deviation over three seeds.}
\label{tab:sweep_kappa}
\scriptsize\setlength{\tabcolsep}{2pt}
\renewcommand{\arraystretch}{0.95}
\setlength{\aboverulesep}{0.7pt}\setlength{\belowrulesep}{0.7pt}
\begin{tabularx}{\linewidth}{l|*{5}{>{\centering\arraybackslash}X}|>{\centering\arraybackslash}X}
\toprule
$\kappa$ & GSM8K & MATH500 & SVAMP & Minerva & Olympiad & Avg \\
\midrule
$0$ & $55.62_{\pm2.69}$ & $20.87_{\pm1.10}$ & $59.78_{\pm2.83}$ & $5.88_{\pm0.37}$ & $2.97_{\pm0.15}$ & $29.02_{\pm1.43}$ \\
$10^{-3}$ & $64.01_{\pm0.65}$ & $24.00_{\pm0.20}$ & $68.67_{\pm0.67}$ & $6.86_{\pm0.21}$ & $3.51_{\pm0.09}$ & $33.41_{\pm0.36}$ \\
\rowcolor{blue!8}
$5\times10^{-3}$ (default) & $67.88_{\pm1.64}$ & $24.53_{\pm1.03}$ & $70.22_{\pm1.35}$ & $5.88_{\pm0.64}$ & $5.39_{\pm0.62}$ & $34.78_{\pm0.25}$ \\
$2\times10^{-2}$ & $65.93_{\pm2.31}$ & $24.73_{\pm0.95}$ & $70.56_{\pm2.46}$ & $7.11_{\pm0.21}$ & $3.61_{\pm0.09}$ & $34.39_{\pm1.20}$ \\
$0.1$ & $64.42_{\pm1.67}$ & $22.60_{\pm4.06}$ & $64.44_{\pm1.95}$ & $4.90_{\pm0.85}$ & $4.50_{\pm0.37}$ & $32.17_{\pm1.28}$ \\
\bottomrule
\end{tabularx}
\end{table}

\subsection{Out-of-Domain Evaluation}
\label{app:experimental-ood}

\begin{table}[H]
\centering
\caption{Out-of-domain accuracy (\%) on Qwen1.5-MoE-A2.7B after RL on GSM8K. Mean $\pm$ standard deviation over three seeds; Science, Code, and All are item-weighted over the knowledge, code, and all benchmarks.}
\label{tab:ood}
\scriptsize\setlength{\tabcolsep}{1.6pt}
\renewcommand{\arraystretch}{0.88}
\setlength{\aboverulesep}{0.7pt}\setlength{\belowrulesep}{0.7pt}
\begin{tabularx}{\linewidth}{l|*{3}{>{\centering\arraybackslash}X}|*{2}{>{\centering\arraybackslash}X}|*{3}{>{\centering\arraybackslash}X}}
\toprule
Method & MMLU & ARC-C & OBQA & HumanEval & MBPP & Science & Code & All \\
\midrule
Base model (no RL) & $47.95$ & $72.61$ & $66.60$ & $44.51$ & $33.40$ & $55.88$ & $36.14$ & $53.49$ \\
No correction & $49.96_{\pm0.36}$ & $72.90_{\pm0.65}$ & $69.67_{\pm0.42}$ & $48.37_{\pm1.86}$ & $33.67_{\pm1.40}$ & $57.58_{\pm0.23}$ & $37.30_{\pm1.52}$ & $55.12_{\pm0.34}$ \\
\cmidrule(lr){1-9}
Exact ratio & $50.49_{\pm2.21}$ & $67.49_{\pm5.67}$ & $69.20_{\pm2.11}$ & $50.00_{\pm1.22}$ & $33.27_{\pm1.21}$ & $56.56_{\pm3.00}$ & $37.40_{\pm1.11}$ & $54.24_{\pm2.56}$ \\
KPop$^\dagger$ & $51.33_{\pm0.23}$ & $71.87_{\pm1.08}$ & $71.07_{\pm1.21}$ & $50.20_{\pm3.57}$ & $33.33_{\pm1.89}$ & $58.36_{\pm0.33}$ & $37.50_{\pm2.19}$ & $55.84_{\pm0.25}$ \\
IcePop & $51.79_{\pm0.40}$ & $72.81_{\pm1.44}$ & $71.40_{\pm2.42}$ & $52.64_{\pm0.35}$ & $35.40_{\pm1.64}$ & $58.93_{\pm0.10}$ & $39.66_{\pm1.30}$ & $56.60_{\pm0.09}$ \\
TIS & $48.76_{\pm1.81}$ & $69.77_{\pm6.27}$ & $70.87_{\pm1.79}$ & $48.17_{\pm1.22}$ & $33.47_{\pm1.60}$ & $56.15_{\pm2.85}$ & $37.10_{\pm1.51}$ & $53.85_{\pm2.32}$ \\
\cmidrule(lr){1-9}
Seq-TIS & $51.43_{\pm0.65}$ & $\mathbf{73.41}_{\pm0.57}$ & $71.00_{\pm1.91}$ & $50.20_{\pm2.75}$ & $\mathbf{36.13}_{\pm1.33}$ & $58.80_{\pm0.56}$ & $39.61_{\pm0.45}$ & $56.48_{\pm0.44}$ \\
Seq-MIS & $50.45_{\pm1.57}$ & $72.35_{\pm0.92}$ & $70.80_{\pm1.59}$ & $50.81_{\pm1.86}$ & $35.87_{\pm0.83}$ & $57.88_{\pm1.28}$ & $39.56_{\pm0.92}$ & $55.66_{\pm1.02}$ \\
GSPO & $50.32_{\pm1.82}$ & $72.61_{\pm1.19}$ & $70.67_{\pm3.29}$ & $44.31_{\pm3.36}$ & $32.53_{\pm1.30}$ & $57.84_{\pm1.72}$ & $35.44_{\pm1.76}$ & $55.13_{\pm1.69}$ \\
FP16 & $49.19_{\pm0.70}$ & $73.27_{\pm0.26}$ & $68.67_{\pm0.42}$ & $48.78_{\pm3.23}$ & $34.47_{\pm1.17}$ & $57.06_{\pm0.49}$ & $38.00_{\pm0.86}$ & $54.75_{\pm0.44}$ \\
\cmidrule(lr){1-9}
\rowcolor{blue!8}
CIS (ours) & $\mathbf{51.87}_{\pm0.25}$ & $72.92_{\pm1.28}$ & $\mathbf{72.13}_{\pm1.75}$ & $\mathbf{53.25}_{\pm2.46}$ & $35.80_{\pm2.40}$ & $\mathbf{59.08}_{\pm0.52}$ & $\mathbf{40.11}_{\pm2.24}$ & $\mathbf{56.79}_{\pm0.71}$ \\
\bottomrule
\end{tabularx}
\end{table}

Table~\ref{tab:ood} evaluates the checkpoints of Table~\ref{tab:main} on Qwen1.5-MoE-A2.7B on three knowledge benchmarks (MMLU-STEM, ARC-Challenge, and OpenBookQA) and two code benchmarks (HumanEval and MBPP), none of which the RL stage uses. Every trained arm ends above the untrained model on the $5{,}489$ pooled items, so RL on GSM8K does not erase knowledge or coding ability. CIS has the highest pooled accuracy ($56.79$), but its margin over IcePop ($56.60$) is smaller than the standard deviations, so we do not claim a separation between the two. The exact ratio and TIS have the widest spreads on the pooled accuracy ($\pm2.56$ and $\pm2.32$), as they do in Table~\ref{tab:main}.

\subsection{Computational Cost}
\label{app:experimental-cost}

\begin{figure}[!htbp]
\centering
\begin{minipage}[t]{0.655\linewidth}\centering
\includegraphics[width=\linewidth]{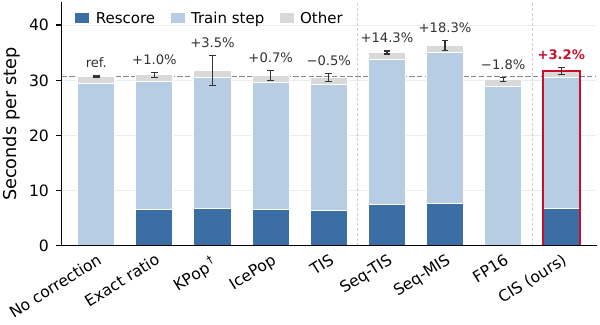}\\[-2pt]{\small (a)}
\end{minipage}\hfill
\begin{minipage}[t]{0.325\linewidth}\centering
\includegraphics[width=\linewidth]{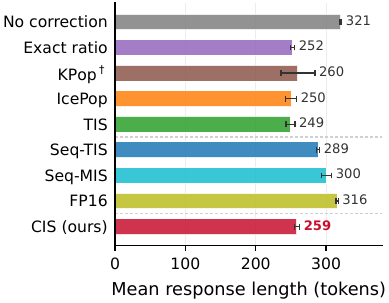}\\[-2pt]{\small (b)}
\end{minipage}
\caption{Compute time per step on Qwen1.5-MoE-A2.7B, excluding checkpointing and evaluation. (a) Time by phase; dashed: no correction. (b) Mean response length.}
\label{fig:cost}
\end{figure}

\begin{table}[!htbp]
\centering
\caption{Compute time per optimizer step (seconds) on eight A100 GPUs, mean $\pm$ standard deviation over three seeds. Rescore is the training-engine forward pass that produces the training-side log-probabilities.}
\label{tab:cost}
\scriptsize\setlength{\tabcolsep}{2pt}
\renewcommand{\arraystretch}{0.95}
\setlength{\aboverulesep}{0.7pt}\setlength{\belowrulesep}{0.7pt}
\begin{tabularx}{\linewidth}{l|*{7}{>{\centering\arraybackslash}X}}
\toprule
Method & Rollout & Rescore & Train step & Advantage & Total & vs.\ no corr. & Length \\
\midrule
No correction & $0.44_{\pm0.02}$ & $0.00_{\pm0.00}$ & $29.41_{\pm0.08}$ & $0.82_{\pm0.02}$ & $30.66_{\pm0.09}$ & --- & $321_{\pm1}$ \\
Exact ratio & $0.48_{\pm0.03}$ & $6.50_{\pm0.08}$ & $23.26_{\pm0.40}$ & $0.72_{\pm0.01}$ & $30.97_{\pm0.48}$ & $+1.0\%$ & $252_{\pm3}$ \\
KPop$^\dagger$ & $0.45_{\pm0.01}$ & $6.64_{\pm0.58}$ & $23.90_{\pm2.08}$ & $0.74_{\pm0.03}$ & $31.73_{\pm2.68}$ & $+3.5\%$ & $260_{\pm24}$ \\
IcePop & $0.49_{\pm0.03}$ & $6.51_{\pm0.23}$ & $23.12_{\pm0.67}$ & $0.77_{\pm0.06}$ & $30.88_{\pm0.93}$ & $+0.7\%$ & $250_{\pm8}$ \\
TIS & $0.48_{\pm0.03}$ & $6.40_{\pm0.13}$ & $22.89_{\pm0.56}$ & $0.74_{\pm0.01}$ & $30.51_{\pm0.69}$ & $-0.5\%$ & $249_{\pm6}$ \\
\midrule
Seq-TIS & $0.49_{\pm0.02}$ & $7.39_{\pm0.04}$ & $26.44_{\pm0.24}$ & $0.72_{\pm0.01}$ & $35.03_{\pm0.28}$ & $+14.3\%$ & $289_{\pm2}$ \\
Seq-MIS & $0.49_{\pm0.00}$ & $7.61_{\pm0.23}$ & $27.44_{\pm0.66}$ & $0.74_{\pm0.01}$ & $36.28_{\pm0.89}$ & $+18.3\%$ & $300_{\pm7}$ \\
FP16 & $0.46_{\pm0.03}$ & $0.00_{\pm0.00}$ & $28.81_{\pm0.34}$ & $0.84_{\pm0.05}$ & $30.11_{\pm0.40}$ & $-1.8\%$ & $316_{\pm2}$ \\
\midrule
\rowcolor{blue!8}
CIS (ours) & $0.49_{\pm0.03}$ & $6.67_{\pm0.13}$ & $23.77_{\pm0.48}$ & $0.72_{\pm0.01}$ & $31.65_{\pm0.62}$ & $+3.2\%$ & $259_{\pm4}$ \\
\bottomrule
\end{tabularx}
\end{table}

Every operator is an elementwise map of the rollout-side and training-side log-probabilities, so no operator adds a forward or backward pass relative to the others. The arms that use the decoupled loss run one rescoring pass of the training engine per step (6.4 to 7.6 seconds), and the train step of these arms is shorter by a similar amount; the uncorrected and fp16 arms do not rescore. Figure~\ref{fig:cost} and Table~\ref{tab:cost} report the resulting compute time per step. The token-level operators lie between $-0.5\%$ and $+3.5\%$ of the uncorrected run, and CIS costs $3.2\%$ more. Seq-TIS and Seq-MIS cost $14\%$ and $18\%$ more, but they also generate longer responses (289 and 300 tokens against 249 to 260 for the token-level operators), and response length changes the step time independently of the operator. Normalizing by the number of tokens does not remove this confound either, because the uncorrected run processes the most tokens per step and is therefore the fastest per token. The wall-clock time between steps also includes checkpointing and evaluation, which differ between runs by up to 25 seconds per step because of how often each run was scheduled to evaluate, not because of the operator; we exclude them here.

\subsection{Where Each Operator Acts}
\label{app:experimental-masking}\label{app:experimental-results}

\begin{figure}[!htbp]
\centering
\includegraphics[width=0.8\linewidth]{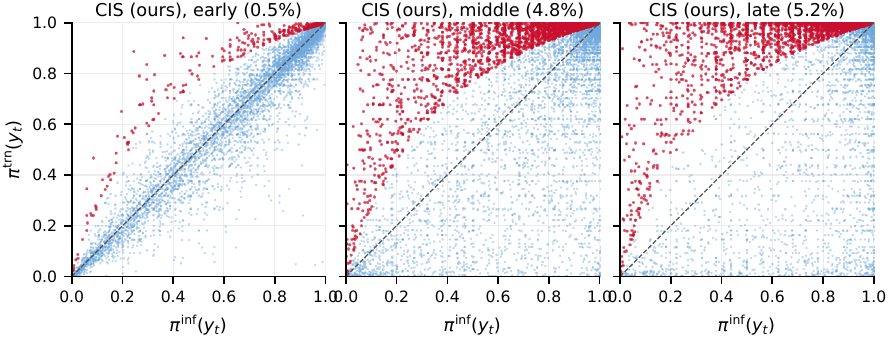}\\[-1pt]{\small (a)}\\[4pt]
\includegraphics[width=0.8\linewidth]{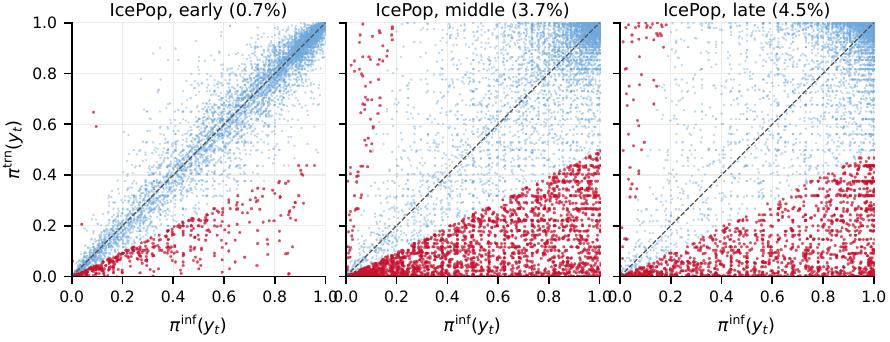}\\[-1pt]{\small (b)}\\[4pt]
\includegraphics[width=0.8\linewidth]{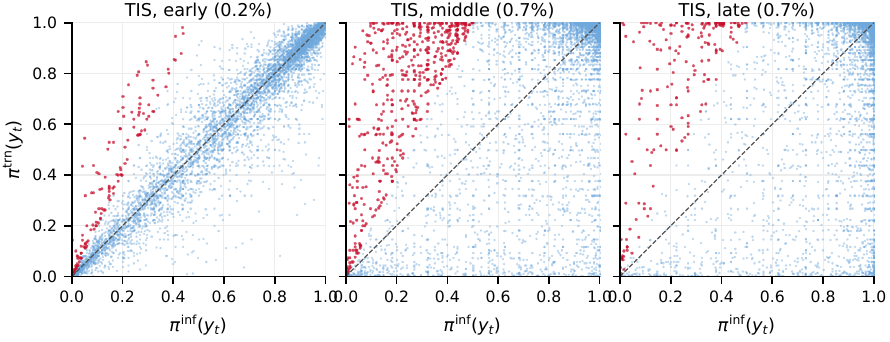}\\[-1pt]{\small (c)}\\[4pt]
\includegraphics[width=0.8\linewidth]{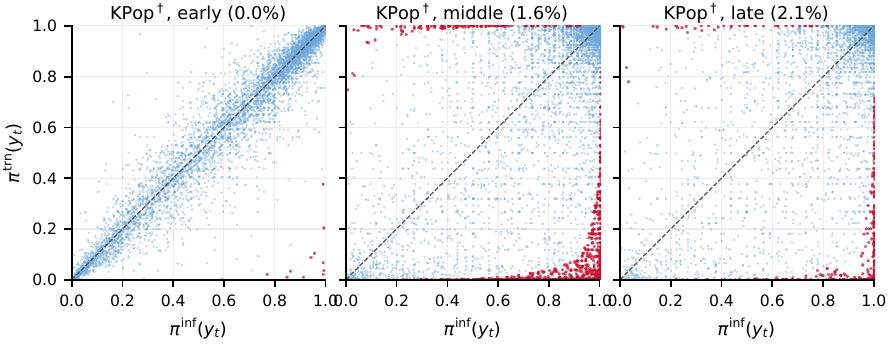}\\[-1pt]{\small (d)}
\caption{Tokens kept (blue) and truncated or masked (red) in the run of each operator (seed two), early, middle, and late in training; panel titles give the share acted on. (a) CIS. (b) IcePop. (c) TIS. (d) KPop$^\dagger$.}
\label{fig:masking}
\end{figure}

Figure~\ref{fig:masking} shows which tokens each operator removes or down-weights during training. Early, middle, and late are the first, middle, and last twentieth of training; kept tokens are subsampled to $25{,}000$ per panel, and between $68\%$ and $85\%$ of all tokens have $\pi^{\mathrm{trn}}(y_t)>0.99$ and overlap in the upper right corner. Early in training the two probabilities agree closely and every operator acts on less than $1\%$ of tokens. As the policy sharpens, the spread around the diagonal grows. The operators then act on different regions of the plane. CIS truncates only above the diagonal, and its truncated tokens concentrate at $\pi^{\mathrm{trn}}(y_t)\to1$, where the training engine is confident and the displacement exceeds the threshold. This concentration does not mean that the floor does most of the work. Applying \Cref{alg:cis} to the logged tokens of the three CIS runs, about one third of the truncated tokens lie below the floor, where the cap equals $1+\lambda\kappa\approx1.01$, and these tokens account for only about $10\%$ of the total truncated excess $\sum_t(k_t-f_t)$; more than half of it comes from tokens with $1-p_t\ge0.1$. IcePop masks mainly below the diagonal, on tokens whose ratio falls under $0.5$, which the one-sided band of CIS leaves untouched. TIS caps the tokens whose training-side probability exceeds twice the inference-side probability, which lie mostly at low $\pi^{\mathrm{inf}}(y_t)$. KPop$^\dagger$ masks tokens near the edges of the plane, where one engine assigns a probability close to one and the other does not, at a rate that grows to $2.1\%$ late in training. CIS and IcePop therefore act on complementary sets of tokens.

\noindent\textbf{The KL mask and probability saturation.}
The binary-KL metric of the KPop baseline is $\mathrm{KL}(\mathrm{Bern}(p)\,\|\,\mathrm{Bern}(q))=p\log(p/q)+(1-p)\log((1-p)/(1-q))$ evaluated in both directions, with $p$ and $q$ clamped to $[10^{-8},1-10^{-8}]$. In single precision the upper clamp is the identity, so a token whose probability is exactly one on both engines yields $0\cdot\log(0/0)$, and a token whose probability is exactly one on one engine yields $+\infty$; the bounds check treats both as out of bounds. On the tokens logged in the KPop run, $16.5\%$ have probability exactly one on both engines and $2.8\%$ on one engine, which together with the $0.45\%$ of tokens whose finite metric exceeds the threshold reproduces the $22\%$ intervention rate that the framework reports. Recomputing the metric in double precision with a representable clamp and mapping the undefined case to zero gives a mask rate of $0.45\%$. The main text reports this corrected variant as KPop$^\dagger$, so that the comparison is not driven by probability saturation.\Cref{fig:masking} counts tokens, so the truncations of CIS appear
concentrated at high confidence: $83\%$ of the tokens that CIS truncates have
$p_t\ge0.9$, compared with $25\%$ under TIS. By removed weight, however, more
than half of the truncation of CIS comes from tokens with $p_t<0.9$.

\subsection{Limitations}
\label{app:scope}\label{app:limitations}

\noindent\textbf{Models and engines.}
We train only mixture-of-experts models, on which the mismatch has a heavy tail (\Cref{app:mismatch_anatomy}). On the dense control the displacement is concentrated near zero, so we expect every correction to matter less there, but we do not train a dense model. The measurements and the training runs use bf16 and specific inference and training engines; other kernels or numerical formats can change the size of the mismatch, in which case $\lambda$ may need to be retuned. In asynchronous training, $k_t$ also contains the policy lag of stale rollouts, which our analysis does not model separately; stale tokens make up less than $1\%$ of the tokens in our runs.

\noindent\textbf{Hyperparameters and baselines.}
The threshold $\lambda$ and the floor $\kappa$ are chosen from a preregistered grid on Qwen1.5-MoE-A2.7B and kept fixed for the other two models, whereas every baseline uses the default hyperparameters of its reference implementation. A comparable sweep for each baseline, and combinations of CIS with methods that reduce the mismatch at its source, such as routing replay~\citep{ma2025r3} or FP16 training~\citep{qi2025fp16}, are left to future work.

\noindent\textbf{Scope of the theory.}
\Cref{thm:exact_risk,thm:cis_risk} give upper bounds under bounded per-token gradients and independent token contributions; for correlated contributions, the sample size $n$ is replaced by $n_{\mathrm{eff}}=n/(1+(n-1)\bar\rho)$ (\Cref{app:proofs}). Any cap with a finite maximum admits a bound similar to \Cref{thm:cis_risk}, so the theory explains why truncation is needed but does not by itself rank CIS against other bounded caps; that comparison rests on the experiments of \Cref{sec:experiments}.

\noindent\textbf{What the token-level correction does not correct.}
The target gradient $G_0$ corrects the conditional distribution of each token given its prefix and not the distribution of the prefix itself. Neither the exact ratio nor CIS removes the shift in the prefix distribution. Seq-TIS and Seq-MIS act on the product of the ratios over a response, which addresses the prefix, but the product leaves any fixed interval for most responses of a few hundred tokens: Seq-MIS rejects $95\%$ of the responses in our runs (Table~\ref{tab:train_stats}), and Seq-TIS caps $12$ to $17\%$ of them.

\end{document}